\documentclass[a4paper,11pt,reqno]{amsart}

\usepackage{graphicx,natbib,amssymb,datetime,float,MnSymbol,currfile}
\usepackage{tikz}
\usetikzlibrary{arrows}
\usetikzlibrary{decorations.markings}

\newtheorem{theorem}{Theorem}
\newtheorem{lemma}{Lemma}

\def\ci{\!\perp\!}

\def\ra{\rightarrow}
\def\la{\leftarrow}
\def\aa{\leftrightarrow}

\def\ot{\mathrel{\reflectbox{\ensuremath{\multimap}}}}
\def\to{\multimap}
\newcommand{\comments}[1]{}

\tikzset{tt/.style={decoration={
  markings,
  mark=at position .485 with {\arrow{>}},
  mark=at position .515 with {\arrow{<}}},postaction={decorate}}}

\begin{document}

\title[]{Alternative Markov and Causal Properties for Acyclic Directed Mixed Graphs}

\author[]{Jose M. Pe\~{n}a\\
ADIT, IDA, Link\"oping University, SE-58183 Link\"{o}ping, Sweden\\
jose.m.pena@liu.se}

\date{\currenttime, \ddmmyydate{\today}, \currfilename}

\begin{abstract}
We extend Andersson-Madigan-Perlman chain graphs by (i) relaxing the semidirected acyclity constraint so that only directed cycles are forbidden, and (ii) allowing up to two edges between any pair of nodes. We introduce global, and ordered local and pairwise Markov properties for the new models. We show the equivalence of these properties for strictly positive probability distributions. We also show that when the random variables are continuous, the new models can be interpreted as systems of structural equations with correlated errors. This enables us to adapt Pearl's {\em do}-calculus to them. Finally, we describe an exact algorithm for learning the new models from observational and interventional data via answer set programming.
\end{abstract}

\maketitle

\section{\bf Introduction}

Chain graphs (CGs) are graphs with possibly directed and undirected edges but without semidirected cycles. They have been extensively studied as a formalism to represent probabilistic independence models, because they can model symmetric and asymmetric relationships between random variables. Moreover, they are much more expressive than directed acyclic graphs (DAGs) and undirected graphs (UGs) \citep{SonntagandPenna2016}. There are three different interpretations of CGs as independence models: The Lauritzen-Wermuth-Frydenberg (LWF) interpretation \citep{Lauritzen1996}, the multivariate regression (MVR) interpretation \citep{CoxandWermuth1996}, and the Andersson-Madigan-Perlman (AMP) interpretation \citep{Anderssonetal.2001}. No interpretation subsumes another \citep{Anderssonetal.2001,SonntagandPenna2015}. Moreover, AMP and MVR CGs are coherent with data generation by block-recursive normal linear regressions \citep{Anderssonetal.2001}.

\citet{Richardson2003} extends MVR CGs by (i) relaxing the semidirected acyclity constraint so that only directed cycles are forbidden, and (ii) allowing up to two edges between any pair of nodes. The resulting models are called acyclic directed mixed graphs (ADMGs). These are the models in which Pearl's {\em do}-calculus operate to determine if the causal effect of an intervention is identifiable from observed quantities \citep{Pearl2009}. In this paper, we make the same two extensions to AMP CGs. We call our ADMGs alternative as opposed to the ones proposed by Richardson, which we call original. It is worth mentioning that neither the original ADMGs nor any other family of mixed graphical models that we know of (e.g. summary graphs \citep{CoxandWermuth1996}, ancestral graphs \citep{RichardsonandSpirtes2002}, MC graphs \citep{Koster2002} or loopless mixed graphs \citep{SadeghiandLauritzen2014}) subsume AMP CGs and hence our alternative ADMGs. To see it, we refer the reader to the works by \citet[p. 1025]{RichardsonandSpirtes2002} and \citet[Section 4.1]{SadeghiandLauritzen2014}. Therefore, our work complements the existing works.

The rest of the paper is organized as follows. Section \ref{sec:preliminaries} introduces some preliminaries. Sections \ref{sec:global} and \ref{sec:local} introduce global, and ordered local and pairwise Markov properties for our ADMGs, and prove their equivalence. When the random variables are continuous, Section \ref{sec:interpretation} offers an intuitive interpretation of our ADMGs as systems of structural equations with correlated errors, so that Pearl's {\em do}-calculus can easily be adapted to them. Section \ref{sec:asp} describes an exact algorithm for learning our ADMGs from observational and interventional data via answer set programming \citep{gelfond_1988,DBLP:journals/amai/Niemela99,DBLP:journals/ai/SimonsNS02}. We close the paper with some discussion in Section \ref{sec:discussion}. Formal proofs of the claims made in this paper can be found in the appendix.

\section{\bf Preliminaries}\label{sec:preliminaries}

In this section, we introduce some concepts about graphical models. Unless otherwise stated, all the graphs and probability distributions in this paper are defined over a finite set $V$. The elements of $V$ are not distinguished from singletons. An ADMG $G$ is a graph with possibly directed and undirected edges but without directed cycles. There may be up to two edges between any pair of nodes, but in that case the edges must be different and one of them must be undirected to avoid directed cycles. Edges between a node and itself are not allowed. See Figure \ref{fig:example} for two examples of ADMGs.

\begin{figure}
\centering
\begin{tabular}{|c|c|}
\hline
\begin{tikzpicture}[inner sep=1mm]
\node at (0,0) (A) {$A$};
\node at (1,0) (B) {$B$};
\node at (2,0) (D) {$D$};

\path[->] (A) edge (B);
\path[->] (B) edge (D);
\path[-] (B) edge [bend left] (D);
\end{tikzpicture}
&
\begin{tikzpicture}[inner sep=1mm]
\node at (0,0) (A) {$A$};
\node at (1,0) (B) {$B$};
\node at (2,0) (C) {$C$};
\node at (3,0) (D) {$D$};

\path[->] (A) edge (B);
\path[->] (B) edge (C);
\path[->] (C) edge (D);
\path[-] (A) edge [bend left] (C);
\path[-] (B) edge [bend left] (D);
\end{tikzpicture}\\
\hline
\end{tabular}\caption{Examples of ADMGs.}\label{fig:example}
\end{figure}
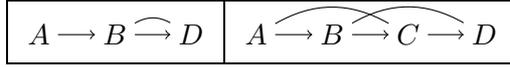

Given an ADMG $G$, we represent with $A \to B$ that $A \ra B$ or $A - B$ (or both) is in $G$. The parents of $X \subseteq V$ in $G$ are $Pa_G(X) = \{A | A \ra B$ is in $G$ with $B \in X \}$. The children of $X$ in $G$ are $Ch_G(X) = \{A | A \la B$ is in $G$ with $B \in X \}$. The neighbours of $X$ in $G$ are $Ne_G(X) = \{A | A - B$ is in $G$ with $B \in X \}$. The ancestors of $X$ in $G$ are $An_G(X) = \{A | A \ra \ldots \ra B$ is in $G$ with $B \in X$ or $A \in X \}$. The descendants of $X$ in $G$ are $De_G(X) = \{A | A \la \ldots \la B$ is in $G$ with $B \in X$ or $A \in X \}$. The semidescendants of $X$ in $G$ are $de_G(X) = \{A | A \ot \ldots \ot B$ is in $G$ with $B \in X$ or $A \in X \}$. The non-semidescendants of $X$ in $G$ are $Nd_G(X) = V \setminus de_G(X)$. The connectivity component of $X$ in $G$ is $Cc_G(X) = \{A | A - \ldots - B$ is in $G$ with $B \in X$ or $A \in X \}$. The connectivity components in $G$ are denoted as $Cc(G)$. A route between a node $V_{1}$ and a node $V_{n}$ on $G$ is a sequence of (not necessarily distinct) nodes $V_{1}, \ldots, V_{n}$ such that $V_i$ and $V_{i+1}$ are adjacent in $G$ for all $1 \leq i < n$. If the nodes in the route are all distinct, then the route is called a path. Finally, the subgraph of $G$ induced by $X \subseteq V$, denoted as $G_X$, is the graph over $X$ that has all and only the edges in $G$ whose both ends are in $X$.

Every probability distribution $p$ satisfies the following four properties, where $X$, $Y$, $W$ and $Z$ disjoint subsets of $V$: Symmetry $X \ci_p Y | Z \Rightarrow Y \ci_p X | Z$, decomposition $X \ci_p Y \cup W | Z \Rightarrow X \ci_p Y | Z$, weak union $X \ci_p Y \cup W | Z \Rightarrow X \ci_p Y | Z \cup W$, and contraction $X \ci_p Y | Z \cup W \land X \ci_p W | Z \Rightarrow X \ci_p Y \cup W | Z$. If $p$ is strictly positive, then it also satisfies the intersection property $X \ci_p Y | Z \cup W \land X \ci_p W | Z \cup Y \Rightarrow X \ci_p Y \cup W | Z$. Some (not yet characterized) probability distributions also satisfy the composition property $X \ci_p Y | Z \land X \ci_p W | Z \Rightarrow X \ci_p Y \cup W | Z$.

\section{\bf Global Markov Property for ADMGs}\label{sec:global}

In this section, we introduce four separation criteria for ADMGs. Moreover, we show that they are all equivalent for strictly positive probability distributions. A probability distribution is said to satisfy the global Markov property with respect to an ADMG if every separation in the graph can be interpreted as an independence in the distribution.

{\bf Criterion 1}. A node $C$ on a path in an ADMG $G$ is said to be a collider on the path if $A \ra C \ot B$ is a subpath. Moreover, the path is said to be connecting given $Z \subseteq V$ when
\begin{itemize}
\item every collider on the path is in $An_G(Z)$, and

\item every non-collider $C$ on the path is outside $Z$ unless $A - C - B$ is a subpath and $Pa_G(C) \setminus Z \neq \emptyset$.
\end{itemize}

Let $X$, $Y$ and $Z$ denote three disjoint subsets of $V$. When there is no path in $G$ connecting a node in $X$ and a node in $Y$ given $Z$, we say that $X$ is separated from $Y$ given $Z$ in $G$ and denote it as $X \ci_G Y | Z$.

{\bf Criterion 2}. A node $C$ on a route in an ADMG $G$ is said to be a collider on the route if $A \ra C \ot B$ is a subroute. Note that maybe $A = B$. Moreover, the route is said to be connecting given $Z \subseteq V$ when
\begin{itemize}
\item every collider on the route is in $Z$, and

\item every non-collider $C$ on the route is outside $Z$.
\end{itemize}

Let $X$, $Y$ and $Z$ denote three disjoint subsets of $V$. When there is no route in $G$ connecting a node in $X$ and a node in $Y$ given $Z$, we say that $X$ is separated from $Y$ given $Z$ in $G$ and denote it as $X \ci_G Y | Z$.

{\bf Criterion 3}. Let $G^u$ denote the UG over $V$ that contains all and only the undirected edges in $G$. The extended subgraph $G[X]$ with $X \subseteq V$ is defined as $G[X] = G_{An_G(X)} \cup (G^u)_{Cc_G(An_G(X))}$. Two nodes $A$ and $B$ in $G$ are said to be collider connected if there is a path between them such that every non-endpoint node is a collider, i.e. $A \ra C \ot B$ or $A \ra C - D \la B$. Such a path is called a collider path. Note that a single edge forms a collider path. The augmented graph $G^a$ is the UG over $V$ such that $A - B$ is in $G^a$ if and only if $A$ and $B$ are collider connected in $G$. The edge $A - B$ is called augmented if it is in $G^a$ but $A$ and $B$ are not adjacent in $G$. A path in $G^a$ is said to be connecting given $Z \subseteq V$ if no node on the path is in $Z$. Let $X$, $Y$ and $Z$ denote three disjoint subsets of $V$. When there is no path in $G[X \cup Y \cup Z]^a$ connecting a node in $X$ and a node in $Y$ given $Z$, we say that $X$ is separated from $Y$ given $Z$ in $G$ and denote it as $X \ci_G Y | Z$.

{\bf Criterion 4}. Given an UG $H$ over $V$ and $X \subseteq V$, we define the marginal graph $H^X$ as the UG over $X$ such that $A - B$ is in $H^X$ if and only if $A - B$ is in $H$ or $A - V_1 - \ldots - V_n - B$ is $H$ with $V_1, \ldots, V_n \notin X$. We define the marginal extended subgraph $G[X]^m$ as $G[X]^m = G_{An_G(X)} \cup ((G^u)_{Cc_G(An_G(X))})^{An_G(X)}$. Let $X$, $Y$ and $Z$ denote three disjoint subsets of $V$. When there is no path in $(G[X \cup Y \cup Z]^m)^a$ connecting a node in $X$ and a node in $Y$ given $Z$, we say that $X$ is separated from $Y$ given $Z$ in $G$ and denote it as $X \ci_G Y | Z$.

The first three separation criteria introduced above coincide with those introduced by \citet{Anderssonetal.2001} and \citet{Levitzetal.2001} for AMP CGs. The equivalence for AMP CGs of these three criteria has been proven by \citet[Theorem 4.1]{Levitzetal.2001}. The following theorems prove the equivalence for ADMGs of the four separation criteria introduced above.

\begin{theorem}\label{the:1}
There is a path in an ADMG $G$ connecting a node in $X$ and a node in $Y$ given $Z$ if and only if there is a path in $G[X \cup Y \cup Z]^a$ connecting a node in $X$ and a node in $Y$ given $Z$.
\end{theorem}

\begin{theorem}\label{the:2}
There is a path in an ADMG $G$ connecting $A$ and $B$ given $Z$ if and only if there is a route in $G$ connecting $A$ and $B$ given $Z$.
\end{theorem}

\begin{theorem}\label{the:3}
Given an ADMG $G$, there is a path in $G[X \cup Y \cup Z]^a$ connecting a node in $X$ and a node in $Y$ given $Z$ if and only if there is a path in $(G[X \cup Y \cup Z]^m)^a$ connecting a node in $X$ and a node in $Y$ given $Z$.
\end{theorem}

Unlike in AMP CGs, two non-adjacent nodes in an ADMG are not necessarily separated. For example, $A \ci_G D | Z$ does not hold for any $Z$ in the ADMGs in Figure \ref{fig:example}. This drawback is shared by the original ADMGs \citep[p. 752]{EvansandRichardson2013}, summary graphs and MC graphs \citep[p. 1023]{RichardsonandSpirtes2002}, and ancestral graphs \citep[Section 3.7]{RichardsonandSpirtes2002}. For ancestral graphs, the problem can be solved by adding edges to the graph without altering the separations represented until every missing edge corresponds to a separation \citep[Section 5.1]{RichardsonandSpirtes2002}. A similar solution does not exist for our ADMGs (we omit the details).

\section{\bf Ordered Local and Pairwise Markov Properties for ADMGs}\label{sec:local}

In this section, we introduce ordered local and pairwise Markov properties for ADMGs. Given an ADMG $G$, the directed acyclity of $G$ implies that we can specify a total ordering ($\prec$) of the nodes of $G$ such that $A \prec B$ only if $B \notin An_G(A)$. Such an ordering is said to be consistent with $G$. Let the predecessors of $A$ with respect to $\prec$ be defined as $Pre_G(A, \prec) = \{ B | B \prec A$ or $B = A \}$. Given $S \subseteq V$, we define the Markov blanket of $B \in S$ with respect to $G[S]$ as $Mb_{G[S]}(B) = Ch_{G[S]}(B) \cup Ne_{G[S]}(B \cup Ch_{G[S]}(B)) \cup Pa_{G[S]}(B \cup Ch_{G[S]}(B) \cup Ne_{G[S]}(B \cup Ch_{G[S]}(B)))$. We say that a probability distribution $p$ satisfies the ordered local Markov property with respect to $G$ and $\prec$ if for any $A \in V$ and $S \subseteq Pre_G(A, \prec)$ such that $A \in S$
\[
B \ci_p S \setminus ( B \cup Mb_{G[S]}(B) ) | Mb_{G[S]}(B)
\]
for all $B \in S$.

\begin{theorem}\label{the:local}
Given a probability distribution $p$ satisfying the intersection property, $p$ satisfies the global Markov property with respect to an ADMG if and only if it satisfies the ordered local Markov property with respect to the ADMG and a consistent ordering of its nodes.
\end{theorem}

Similarly, we say that a probability distribution $p$ satisfies the ordered pairwise Markov property with respect to $G$ and $\prec$ if for any $A \in V$ and $S \subseteq Pre_G(A, \prec)$ such that $A \in S$
\[
B \ci_p C | V(G[S]) \setminus ( B \cup C )
\]
for all nodes $B, C \in S$ that are not adjacent in $G[S]^a$, and where $V(G[S])$ denotes the nodes in $G[S]$.

\begin{theorem}\label{the:pairwise}
Given a probability distribution $p$ satisfying the intersection property, $p$ satisfies the global Markov property with respect to an ADMG if and only if it satisfies the ordered pairwise Markov property with respect to the ADMG and a consistent ordering of its nodes.
\end{theorem}

For each $A \in V$ and $S \subseteq Pre_G(A, \prec)$ such that $A \in S$, the ordered local Markov property specifies an independence for each $B \in S$. The number of independences to specify can be reduced by noting that $G[S] = G[An_G(S)]$ and, thus, we do not need to consider every set $S \subseteq Pre_G(A, \prec)$ but only those that are ancestral, i.e. those such that $S = An_G(S)$. The number of independences to specify can be further reduced by considering only maximal ancestral sets, i.e. those sets $S$ such that $Mb_{G[S]}(B) \subset Mb_{G[T]}(B)$ for every ancestral set $T$ such that $S \subset T \subseteq Pre_G(A, \prec)$. The independences for the non-maximal ancestral sets follow from the independences for the maximal ancestral sets by decomposition. A characterization of the maximal ancestral sets is possible but notationally cumbersome (we omit the details). All in all, for each node and maximal ancestral set, the ordered local Markov property specifies an independence for each node in the set. This number is greater than for the original ADMGs, where a single independence is specified for each node and maximal ancestral set \citep[Section 3.1]{Richardson2003}. Even fewer independences are needed for the original ADMGs when interpreted as linear causal models \citep[Section 4]{KangandTian2009}. All in all, our ordered local Markov property serves its purpose, namely to identify a subset of the independences in the global Markov property that implies the rest.

Note that \citet[Theorem 3]{Anderssonetal.2001} describe local and pairwise Markov properties for AMP CGs that are equivalent to the global one under the assumption of the intersection and composition properties. Our ordered local and pairwise Markov properties above only require assuming the intersection property. Note that this assumption is also needed to prove similar results for much simpler models such as UGs \citep[Theorem 3.7]{Lauritzen1996}. For AMP CGs, however, we can do better than just using the ordered local and pairwise Markov properties for ADMGs above. Specifically, we introduce in the next section neater local and pairwise Markov properties for AMP CGs under the intersection property assumption. Later on, we will also use them to prove some results for ADMGs.

\subsection{Local and Pairwise Markov Properties for AMP CGs}

\citet[Theorem 2]{Anderssonetal.2001} introduce the following block-recursive Markov property. A probability distribution $p$ satisfies the global Markov property with respect to an AMP CG $G$ if and only if the following three properties hold for all $C \in Cc(G)$:
\begin{itemize}
\item C1: $C \ci_p Nd_G(C) \setminus Cc_G(Pa_G(C)) | Cc_G(Pa_G(C))$.

\item C2: $p(C | Cc_G(Pa_G(C)))$ satisfies the global Markov property with respect to $G_C$.

\item C3$^*$: $D \ci_p Cc_G(Pa_G(C)) \setminus Pa_G(D) | Pa_G(D)$ for all $D \subseteq C$.
\end{itemize}

We simplify the block-recursive Markov property as follows.

\begin{theorem}\label{the:blockamp}
C1, C2 and C3$^*$ hold if and only if the following two properties hold:
\begin{itemize}
\item C1$^*$: $D \ci_p Nd_G(D) \setminus Pa_G(D) | Pa_G(D)$ for all $D \subseteq C$.

\item C2$^*$: $p(C | Pa_G(C))$ satisfies the global Markov property with respect to $G_C$.
\end{itemize}
\end{theorem}

\citet[Theorem 3]{Anderssonetal.2001} also introduce the following local Markov property. A probability distribution $p$ satisfying the intersection and composition properties satisfies the global Markov property with respect to an AMP CG $G$ if and only if the following two properties hold for all $C \in Cc(G)$:
\begin{itemize}
\item L1: $A \ci_p C \setminus (A \cup Ne_G(A)) | Nd_G(C) \cup Ne_G(A)$ for all $A \in C$.

\item L2: $A \ci_p Nd_G(C) \setminus Pa_G(A) | Pa_G(A)$ for all $A \in C$.
\end{itemize}

We introduce below a local Markov property that is equivalent to the global one under the assumption of the intersection property only.

\begin{theorem}\label{the:localamp}
A probability distribution $p$ satisfying the intersection property satisfies the global Markov property with respect to an AMP CG $G$ if and only if the following two properties hold for all $C \in Cc(G)$:
\begin{itemize}
\item L1: $A \ci_p C \setminus (A \cup Ne_G(A)) | Nd_G(C) \cup Ne_G(A)$ for all $A \in C$.

\item L2$^*$: $A \ci_p Nd_G(C) \setminus Pa_G(A \cup S) | S \cup Pa_G(A \cup S)$ for all $A \in C$ and $S \subseteq C \setminus A$.
\end{itemize}
\end{theorem}

Finally, \citet[Theorem 3]{Anderssonetal.2001} also introduce the following pairwise Markov property. A probability distribution $p$ satisfying the intersection and composition properties satisfies the global Markov property with respect to an AMP CG $G$ if and only if the following two properties hold for all $C \in Cc(G)$:
\begin{itemize}
\item P1: $A \ci_p B | Nd_G(C) \cup C \setminus (A \cup B)$ for all $A \in C$ and $B \in C \setminus (A \cup Ne_G(A))$.

\item P2: $A \ci_p B | Nd_G(C) \setminus B$ for all $A \in C$ and $B \in Nd_G(C) \setminus Pa_G(A)$.
\end{itemize}

We introduce below a pairwise Markov property that is equivalent to the global one under the assumption of the intersection property only.

\begin{theorem}\label{the:pairwiseamp}
A probability distribution $p$ satisfying the intersection property satisfies the global Markov property with respect to an AMP CG $G$ if and only if the following two properties hold for all $C \in Cc(G)$:
\begin{itemize}
\item P1: $A \ci_p B | Nd_G(C) \cup C \setminus (A \cup B)$ for all $A \in C$ and $B \in C \setminus (A \cup Ne_G(A))$.

\item P2$^*$: $A \ci_p B | S \cup Nd_G(C) \setminus B$ for all $A \in C$, $S \subseteq C \setminus A$ and $B \in Nd_G(C) \setminus Pa_G(A \cup S)$.
\end{itemize}
\end{theorem}

\section{\bf Causal Interpretation of ADMGs}\label{sec:interpretation}

Let us assume that $V$ is normally distributed. In this section, we show that an ADMG $G$ can be interpreted as a system of structural equations with correlated errors. Specifically, the system includes an equation for each $A \in V$, which is of the form
\[
A = \beta_A Pa_G(A) + \epsilon_A
\]
where $\epsilon_A$ denotes the error term. The error terms are represented implicitly in $G$. They can be represented explicitly by magnifying $G$ into the ADMG $G'$ as follows:

\begin{table}[H]
\centering
\scalebox{1.0}{
\begin{tabular}{|ll|}
\hline
1 & Set $G'=G$\\
2 & For each node $A$ in $G$\\
3 & \hspace{0.3cm} Add the node $\epsilon_A$ and the edge $\epsilon_A \ra A$ to $G'$\\
4 & For each edge $A - B$ in $G$\\
5 & \hspace{0.3cm} Replace $A - B$ with the edge $\epsilon_A - \epsilon_B$ in $G'$\\
\hline
\end{tabular}}
\end{table}

The magnification above basically consists in adding the error nodes $\epsilon_A$ to $G$ and connect them appropriately. Figure \ref{fig:example2} shows an example. Note that every node $A \in V$ is determined by $Pa_{G'}(A)$ and that $\epsilon_A$ is determined by $A \cup Pa_{G'}(A) \setminus \epsilon_A$. Let $\epsilon$ denote all the error nodes in $G'$. Formally, we say that $A \in V \cup \epsilon$ is determined by $Z \subseteq V \cup \epsilon$ when $A \in Z$ or $A$ is a function of $Z$. We use $Dt(Z)$ to denote all the nodes that are determined by $Z$. From the point of view of the separations, that a node outside the conditioning set of a separation is determined by the conditioning set has the same effect as if the node were actually in the conditioning set. Bearing this in mind, it is not difficult to see that, as desired, $G$ and $G'$ represent the same separations over $V$. The following theorem formalizes this result.

\begin{figure}
\centering
\begin{tabular}{|c|c|}
\hline
$G$&$G'$\\
\hline
\begin{tikzpicture}[inner sep=1mm]
\node at (0,0) (A) {$A$};
\node at (1,0) (B) {$B$};
\node at (0,-1) (C) {$C$};
\node at (1,-1) (D) {$D$};
\node at (0,-2.5) (E) {$E$};
\node at (1,-2.5) (F) {$F$};
\path[->] (A) edge (B);
\path[->] (A) edge (C);
\path[->] (A) edge (D);
\path[->] (B) edge (D);
\path[-] (C) edge (D);
\path[-] (C) edge (E);
\path[-] (D) edge (F);
\path[-] (E) edge [bend left] (F);
\path[->] (E) edge (F);
\end{tikzpicture}
&
\begin{tikzpicture}[inner sep=1mm]
\node at (0,0) (A) {$A$};
\node at (1,0) (B) {$B$};
\node at (0,-1) (C) {$C$};
\node at (1,-1) (D) {$D$};
\node at (0,-2.5) (E) {$E$};
\node at (1,-2.5) (F) {$F$};
\node at (-1,0) (EA) {$\epsilon_A$};
\node at (2,0) (EB) {$\epsilon_B$};
\node at (-1,-1) (EC) {$\epsilon_C$};
\node at (2,-1) (ED) {$\epsilon_D$};
\node at (-1,-2.5) (EE) {$\epsilon_E$};
\node at (2,-2.5) (EF) {$\epsilon_F$};
\path[->] (EA) edge (A);
\path[->] (EB) edge (B);
\path[->] (EC) edge (C);
\path[->] (ED) edge (D);
\path[->] (EE) edge (E);
\path[->] (EF) edge (F);
\path[->] (A) edge (B);
\path[->] (A) edge (C);
\path[->] (A) edge (D);
\path[->] (B) edge (D);
\path[-] (EC) edge [bend right] (ED);
\path[-] (EC) edge (EE);
\path[-] (ED) edge (EF);
\path[-] (EE) edge [bend left] (EF);
\path[->] (E) edge (F);
\end{tikzpicture}\\
\hline
\end{tabular}\caption{Example of the magnification of an ADMG.}\label{fig:example2}
\end{figure}
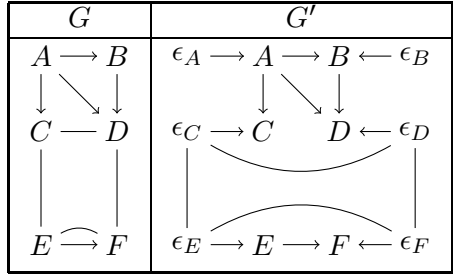

\begin{theorem}\label{the:GG'}
Let $X$, $Y$ and $Z$ denote three disjoint subsets of $V$. Then, $X \ci_G Y | Z$ if and only if $X \ci_{G'} Y | Z$.
\end{theorem}

Finally, let $\epsilon \sim \mathcal{N}(0, \Lambda)$ such that $(\Lambda^{-1})_{\epsilon_A,\epsilon_B} = 0$ if $\epsilon_A - \epsilon_B$ is not in $G'$. Then, $G$ can be interpreted as a system of structural equations with correlated errors as follows. For any $A \in V$
\begin{equation}\label{eq:equation}
A = \beta_A Pa_G(A) + \epsilon_A
\end{equation}
and for any other $B \in V$
\begin{equation}\label{eq:equation2}
covariance(\epsilon_A, \epsilon_B) = \Lambda_{\epsilon_A,\epsilon_B}.
\end{equation}

The following two theorems confirm that the interpretation above works as intended. A similar result to the second theorem exists for the original ADMGs \cite[Theorem 1]{Koster1999}.

\begin{theorem}\label{the:gaussian}
Every probability distribution $p(V)$ specified by Equations (\ref{eq:equation}) and (\ref{eq:equation2}) is Gaussian.
\end{theorem}

\begin{theorem}\label{the:markovian}
Every probability distribution $p(V)$ specified by Equations (\ref{eq:equation}) and (\ref{eq:equation2}) satisfies the global Markov property with respect to $G$.
\end{theorem}

The equations above specify each node as a linear function of its parents with additive normal noise. The equations can be generalized to nonlinear or nonparametric functions as long as the noise remains additive normal. That is, $A = f(Pa_G(A)) + \epsilon_A$ for all $A \in V$, with $\epsilon \sim \mathcal{N}(0, \Lambda)$ such that $(\Lambda^{-1})_{\epsilon_A,\epsilon_B} = 0$ if $\epsilon_A - \epsilon_B$ is not in $G'$. That the noise is additive normal ensures that $\epsilon_A$ is determined by $A \cup Pa_{G'}(A) \setminus \epsilon_A$, which is needed for Theorem \ref{the:GG'} to remain valid which, in turn, is needed for Theorem \ref{the:markovian} to remain valid.

A less formal but more intuitive alternative interpretation of ADMGs is as follows. We can interpret the parents of each node in an ADMG as its observed causes. Its unobserved causes are grouped into an error node that is represented implicitly in the ADMG. We can interpret the undirected edges in the ADMG as the correlation relationships between the different error nodes. The causal structure is constrained to be a DAG, but the correlation structure can be any UG. This causal interpretation of our ADMGs parallels that of the original ADMGs \citep{Pearl2009}. There are however two main differences. First, the noise in the original ADMGs is not necessarily additive normal. Second, the correlation structure of the error nodes in the original ADMGs is represented by a covariance graph, i.e. a graph with only bidirected edges \citep{PearlandWermuth1993}. Therefore, whereas a missing edge between two error nodes in the original ADMGs represents marginal independence, in our ADMGs it represents conditional independence given the rest of the error nodes. This means that the original and our ADMGs represent complementary causal models. Consequently, there are scenarios where the identification of the causal effect of an intervention is not possible with the original ADMGs but is possible with ours, and vice versa. We elaborate on this in the next section.

\subsection{\bf {\em do}-calculus for ADMGs}\label{sec:do}

We start by adapting Pearl's {\em do}-calculus, which operates on the original ADMGs, to our ADMGs. The original {\em do}-calculus consists of the following three rules, whose repeated application permits in some cases to identify (i.e. compute) the causal effect of an intervention from observed quantities:
\begin{itemize}
\item Rule 1 (insertion/deletion of observations):

$p(Y| do(X), Z \cup W) = p(Y| do(X), W)$ if $Y \ci_{G''} Z | X \cup W || X$.

\item Rule 2 (action/observation exchange): 

$p(Y| do(X), do(Z), W) = p(Y| do(X), Z \cup W)$ if $Y \ci_{G''} F_Z | X \cup W \cup Z || X$.

\item Rule 3 (insertion/deletion of actions): 

$p(Y| do(X), do(Z), W) = p(Y| do(X), W)$ if $Y \ci_{G''} F_Z | X \cup W || X$.
\end{itemize}
where $X$, $Y$, $Z$ and $W$ are disjoint subsets of $V$, $G''$ is the original ADMG $G$ augmented with an intervention random variable $F_A$ and an edge $F_A \ra A$ for every $A \in V$, and ``$||X$'' denotes an intervention on X in $G''$, i.e. any edge with an arrowhead into any node in $X$ is removed. See \citet[p. 686]{Pearl1995} for further details and the proof that the rules are sound. Fortunately, the rules also apply to our ADMGs by simply redefining ``$||X$'' appropriately. The proof that the rules are still sound is essentially the same as before. Specifically, ``$||X$'' should now be implemented as follows:
\begin{itemize}
\item Delete all the directed edges pointing to nodes in $X$,

\item for every path $A - V_1 - \ldots - V_n - B$ with $A, B \notin X$ and $V_1, \ldots, V_n \in X$, add the edge $A - B$, and

\item delete all the undirected edges with an endnode in $X$.
\end{itemize}
The first step is the same as an intervention in an original ADMG. The second and third steps of the intervention are best understood in terms of the magnified ADMG $G'$: They correspond to marginalizing the error nodes associated to the nodes in $X$ out of $G'_\epsilon$, the UG that represents the correlation structure of the error nodes. In other words, they replace $G'_\epsilon$ with $(G')^{\epsilon \setminus \epsilon_X}$, the marginal graph of $G'_\epsilon$ over $\epsilon \setminus \epsilon_X$. This makes sense since $\epsilon_X$ is no longer associated to $X$ due to the intervention and, thus, we may want to marginalize it out because it is unobserved. This is exactly what the second and third steps of the intervention imply. To see it, note that the ADMG after the intervention and the magnified ADMG after the intervention represent the same separations over $V$, by Theorem \ref{the:GG'}.

Now, we show that the original and our ADMGs allow for complementary causal reasoning. Specifically, we show an example where our ADMGs allow for the identification of the causal effect of an intervention whereas the original ADMGs do not, and vice versa. Consider the DAG in Figure \ref{fig:example3}, which represents the causal relationships among all the random variables in the domain at hand.\footnote{For instance, the DAG may correspond to the following fictitious domain: $A$ = Smoking, $B$ = Lung cancer, $C$ = Drinking, $U_A$ = Parents' smoking, $U_B$ = Parents' lung cancer, $U_C$ = Parents' drinking, $U$ = Parents' genotype that causes smoking and drinking, $U_S$ = Parents' hospitalization.} However, only $A$, $B$ and $C$ are observed. Moreover, $U_S$ represents selection bias. Although other definitions may exist, we say that selection bias is present if two unobserved causes have a common effect that is omitted from the study but influences the selection of the samples in the study \citep[p. 163]{Pearl2009}. Therefore, the corresponding unobserved causes are correlated in every sample selected. Note that this definition excludes the possibility of an intervention affecting the selection because, in a causal model, unobserved causes do not have observed causes. Note also that our goal is not the identification of the causal effect of an intervention in the whole population but in the subpopulation that satisfies the selection bias criterion.\footnote{For instance, in the fictitious domain in the previous footnote, we are interested in the causal effect that smoking may have on the development of lung cancer for the patients with hospitalized parents.} For causal effect identification in the whole population, see \citet{BareinboimandTian2015}.

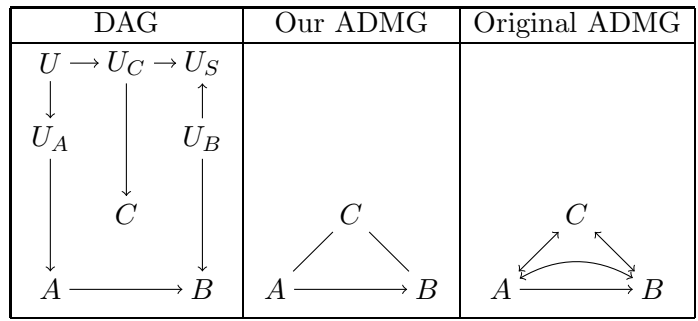
\begin{figure}
\centering
\begin{tabular}{|c|c|c|}
\hline
DAG&Our ADMG&Original ADMG\\
\hline
\begin{tikzpicture}[inner sep=1mm]
\node at (0,0) (A) {$A$};
\node at (2,0) (B) {$B$};
\node at (1,1) (C) {$C$};
\node at (0,2) (UA) {$U_A$};
\node at (2,2) (UB) {$U_B$};
\node at (1,3) (UC) {$U_C$};
\node at (0,3) (U) {$U$};
\node at (2,3) (US) {$U_S$};
\path[->] (A) edge (B);
\path[->] (UA) edge (A);
\path[->] (UB) edge (B);
\path[->] (UC) edge (C);
\path[->] (U) edge (UA);
\path[->] (U) edge (UC);
\path[->] (UB) edge (US);
\path[->] (UC) edge (US);
\end{tikzpicture}
&
\begin{tikzpicture}[inner sep=1mm]
\node at (0,0) (A) {$A$};
\node at (2,0) (B) {$B$};
\node at (1,1) (C) {$C$};
\path[->] (A) edge (B);
\path[-] (A) edge (C);
\path[-] (B) edge (C);
\end{tikzpicture}
&
\begin{tikzpicture}[inner sep=1mm]
\node at (0,0) (A) {$A$};
\node at (2,0) (B) {$B$};
\node at (1,1) (C) {$C$};
\path[->] (A) edge (B);
\path[<->] (A) edge [bend left] (B);
\path[<->] (A) edge (C);
\path[<->] (B) edge (C);
\end{tikzpicture}\\
\hline
\end{tabular}\caption{Example of a case where $p(B | do(A))$ is identifiable with our ADMG but not with the original one.}\label{fig:example3}
\end{figure}

The ADMGs in Figure \ref{fig:example3} represent the causal model represented by the DAG when only the observed random variables are modeled. According to our interpretation of ADMGs above, our ADMG is derived from the DAG by keeping the directed edges between observed random variables, and adding an undirected edge between two observed random variables if and only if their unobserved causes are not separated in the DAG given the unobserved causes of the rest of the observed random variables. In other words, $U_A \ci U_B | U_C$ holds in the DAG but $U_A \ci U_C | U_B$ and $U_B \ci U_C | U_A$ do not and, thus, the edges $A - C$ and $B - C$ are added to the ADMG but $A - B$ is not. Deriving the original ADMG is less straightforward. The bidirected edges in an original ADMG represent potential marginal dependence due to a common unobserved cause, also known as confounding. Thus, the original ADMGs are not meant to model selection bias. The best we can do is then to use bidirected edges to represent potential marginal dependences regardless of their origin. This implies that we can derive the original ADMG from the DAG by keeping the directed edges between observed random variables, and adding a bidirected edge between two observed random variables if and only if their unobserved causes are not separated in the DAG given the empty set. Clearly, $p(B | do(A))$ is not identifiable with the original ADMG but is identifiable with our ADMG. Specifically,
\[
p(B | do(A)) = \sum_C p(B | do(A), C) p(C | do(A)) = \sum_C p(B | do(A), C) p(C) = \sum_C p(B | A, C) p(C)
\]
where the first equality is due to marginalization, the second due to Rule 3, and the third due to Rule 2.

The original ADMGs assume that confounding is always the source of correlation between unobserved causes. In the example above, we consider selection bias as an additional source. However, this is not the only possibility. For instance, $U_B$ and $U_C$ may be tied by a physical law of the form $f(U_B, U_C) = constant$ devoid of causal meaning, much like Boyle's law relates the pressure and volume of a gas as $pressure \cdot volume = constant$ if the temperature and amount of gas remain unchanged within a closed system. In such a case, the discussion above still applies and our ADMG allows for causal effect identification but the original does not. For an example where the original ADMGs allow for causal effect identification whereas ours do not, simply replace the subgraph $U_C \ra U_S \la U_B$ in Figure \ref{fig:example3} with $U_C \la W \ra U_B$ where $W$ is an unobserved random variable. Then, our ADMG will contain the same edges as before plus the edge $A - B$, making the causal effect non-identifiable. The original ADMG will contain the same edges as before with the exception of the edge $A \aa B$, making the causal effect identifiable.

In summary, the bidirected edges of the original ADMGs have a clear semantics: They represent potential marginal dependence due to a common unobserved cause. This means that we have to know the causal relationships between the unobserved random variables to derive the ADMG. Or at least, we have to know that there is no selection bias or tying laws so that marginal dependence can be attributed to a common unobserved cause due to Reichenbach's principle \citep[p. 30]{Pearl2009}. This knowledge may not be available in some cases. Moreover, the original ADMGs are not meant to represent selection bias or tying laws. To solve these two problems, we may be willing to use the bidirected edges to represent potential marginal dependences regardless of their origin. Our ADMGs are somehow dual to the original ADMGs, since the undirected edges represent potential saturated conditional dependence between unobserved causes. This implies that in some cases, such as in the example above, our ADMGs may allow for causal effect identification whereas the original may not.

\section{\bf Learning ADMGs Via ASP}\label{sec:asp}

In this section, we introduce an exact algorithm for learning ADMGs via answer set programming (ASP), which is a declarative constraint satisfaction paradigm that is well-suited for solving computationally hard combinatorial problems \citep{gelfond_1988,DBLP:journals/amai/Niemela99,DBLP:journals/ai/SimonsNS02}. ASP represents constraints in terms of first-order logical rules. Therefore, when using ASP, the first task is to model the problem at hand in terms of rules so that the set of solutions implicitly represented by the rules corresponds to the solutions of the original problem. One or multiple solutions of the original problem can then be obtained by invoking an off-the-shelf ASP solver on the constraint declaration. The algorithms underlying the ASP solver \texttt{clingo} \citep{DBLP:journals/aicom/GebserKKOSS11}, which we use in this work, are based on state-of-the-art Boolean satisfiability solving techniques \citep{DBLP:series/faia/2009-185}.

Figure \ref{fig:asp} shows the ASP encoding of our learning algorithm. The predicate \texttt{node(X)} in rule 1 represents that $X$ is a node. The predicates \texttt{line(X,Y,I)} and \texttt{arrow(X,Y,I)} represent that there is an undirected and directed edge from $X$ to $Y$ after having intervened on the node $I$. The observational regime corresponds to $I=0$. The rules 2-3 encode a non-deterministic guess of the edges for the observational regime, which means that the ASP solver with implicitly consider all possible graphs during the search, hence the exactness of the search. The edges under the observational regime are used in the rules 4-5 to define the edges in the graph after having intervened on $I$, following the description in Section \ref{sec:do}. Therefore, the algorithm assumes continuous random variables and additive normal noise when the input contains interventions. It makes no assumption though when the input consists of just observations. The rules 6-7 enforce the fact that undirected edges are symmetric and that there is at most one directed edge between two nodes. The predicate \texttt{ancestor(X,Y)} represents that $X$ is an ancestor of $Y$. The rules 8-10 enforce that the graph has no directed cycles. The predicates in the rules 11-12 represent whether a node $X$ is or is not in a set of nodes $C$. The rules 13-24 encode the separation criterion 2 in Section \ref{sec:global}. The predicate \texttt{con(X,Y,C,I)} in rules 25-28 represents that there is a connecting route between $X$ and $Y$ given $C$ after having intervened on $I$. The rule 29 enforces that each dependence in the input must correspond to a connecting route. The rule 30 represents that each independence in the input that is not represented implies a penalty of $W$ units. In our case, $W=1$. The rules 31-33 represent a penalty of 1 unit per edge. Other penalty rules can be added similarly.

\begin{figure}
\centering
\tiny
\fbox{
\begin{minipage}{0.7\textwidth}
\input{ASPalternativeADMGsinterventional.tex}
\end{minipage}
}
\caption{ASP encoding of the learning algorithm.}\label{fig:asp}
\end{figure}

Figure \ref{fig:asp} shows the ASP encoding of all the (in)dependences in the probability distribution at hand, e.g. as determined by some available data. Specifically, the predicate \texttt{nodes(3)} represents that there are three nodes in the domain at hand, and the predicate \texttt{set(0..7)} represents that there are eight sets of nodes, indexed from 0 (empty set) to 7 (full set). The predicate \texttt{indep(X,Y,C,I,W)} (respectively \texttt{dep(X,Y,C,I,W)}) represents that the nodes $X$ and $Y$ are conditionally independent (respectively dependent) given the set index $C$ after having intervened on the node $I$. Observations correspond to $I=0$. The penalty for failing to represent an (in)dependence is $W$. Note that it suffices to specify all the (in)dependences between pair of nodes, because these identify uniquely the rest of the independences in the probability distribution \citep[Lemma 2.2]{Studeny2005}. Note also that we do not assume that the probability distribution at hand is faithful to some ADMG or satisfies the composition property, as it is the case in most heuristic learning algorithms.

By calling the ASP solver with the encodings of the learning algorithm and the (in)dependences in the domain, the solver will essentially perform an exhaustive search over the space of graphs, and will output the graphs with the smallest penalty. Specifically, when only the observations are used (i.e. the last six lines of Figure \ref{fig:asp2} are removed), the learning algorithm finds 37 optimal models. Among them, we have UGs such as \texttt{line(1,2) line(2,3) line(3,1)}, DAGs such as \texttt{arrow(3,1) arrow(1,2) arrow(3,2)}, AMP CGs such as \texttt{line(1,2) arrow(3,1) arrow(3,2)}, and ADMGs such as \texttt{line(1,2) line(2,3) arrow(1,2)} or \texttt{line(1,2) line(1,3) arrow(2,3)}. When all the observations and interventions available are used, the learning algorithm finds 18 optimal models. These are the models out the 37 models found before that have no directed edge coming out of the node 3. This is the expected result given the last four lines in Figure \ref{fig:asp2}. Note that the output still includes the ADMGs mentioned before.

Finally, the ASP code easily accommodates prior knowledge in the form of a node ordering, or forbidden and compelled edges. For instance, we can constrain the graphs to be consistent with the ordering $1 \prec 2 \prec 3$ by simply adding the rules in Figure \ref{fig:asp3} (top). Moreover, the ASP code can easily be extended as shown in Figure \ref{fig:asp3} (bottom) to learn not only our ADMGs but also original ADMGs. Note that the second line forbids graphs with both undirected and bidirected edges. This results in 34 optimal models: The 18 previously found plus 16 original ADMG, e.g. \texttt{biarrow(1,2) biarrow(1,3) arrow(1,2)} or \texttt{biarrow(1,2) biarrow(1,3) arrow(2,3)}.

\begin{figure}
\centering
\tiny
\fbox{
\begin{minipage}{0.35\textwidth}
\input{completeinterventional.tex}
\end{minipage}
}
\caption{ASP encoding of the (in)dependences in the domain.}\label{fig:asp2}
\end{figure}

\begin{figure}
\centering
\tiny
\fbox{
\begin{minipage}{0.15\textwidth}
\input{complete2.tex}
\end{minipage}
}
\\
\fbox{
\begin{minipage}{0.6\textwidth}
\input{ASPoriginalADMGs.tex}
\end{minipage}
}
\caption{Additional ASP encoding for adding a node ordering (top) and learning original ADMGs as well as ours (bottom).}\label{fig:asp3}
\end{figure}

\section{\bf Discussion}\label{sec:discussion}

In this work, we have introduced ADMGs as an extension AMP CGs by (i) relaxing the semidirected acyclity constraint so that only directed cycles are forbidden, and (ii) allowing up to two edges between any pair of nodes. We have introduced and proved the equivalence of global, and ordered local and pairwise Markov properties for the new models. We have also shown that when the random variables are continuous, the new models can be interpreted as systems of structural equations with correlated errors. This has enabled us to adapt Pearl's {\em do}-calculus to them. We have shown that our models complement those used in Pearl's {\em do}-calculus, as there are cases where the identification of the causal effect of an intervention is not possible with the latter but is possible with the former, and vice versa. Finally, we have described an exact algorithm for learning the new models from observational and interventional data.

In the future, we plan to unify the original and our ADMGs by allowing directed, undirected and bidirected edges.

\section*{\bf Appendix: Proofs}\label{sec:appendix}

\begin{lemma}\label{lem:ancestors}
If there is a path $\rho$ in an ADMG $G$ between $A \in X$ and $B \in Y$ such that (i) no non-collider $C$ on $\rho$ is in $Z$ unless $A - C - B$ is a subpath of $\rho$ and $Pa_G(C) \setminus Z \neq \emptyset$, and (ii) every collider on $\rho$ is in $An_G(X \cup Y \cup Z)$, then there is a path in $G$ connecting a node in $X$ and a node in $Y$ given $Z$.
\end{lemma}

\begin{proof}
Suppose that $\rho$ has a collider $C$ such that $C \in An_G(D) \setminus An_G(Z)$ with $D \in X$, or $C \in An_G(E) \setminus An_G(Z)$ with $E \in Y$. Assume without loss of generality that $C \in An_G(D) \setminus An_G(Z)$ with $D \in X$ because, otherwise, a symmetric argument applies. Then, replace the subpath of $\rho$ between $A$ and $C$ with $D \la \ldots \la C$. Note that the resulting path (i) has no non-collider in $Z$ unless $A - C - B$ is a subpath of $\rho$ and $Pa_G(C) \setminus Z \neq \emptyset$, and (ii) has every collider in $An_G(X \cup Y \cup Z)$. Note also that the resulting path has fewer colliders than $\rho$ that are not in $An_G(Z)$. Continuing with this process until no such collider $C$ exists produces the desired result.
\end{proof}

\begin{lemma}\label{lem:colliders}
Given an ADMG $G$, let $\rho$ denote a shortest path in $G[A \cup B \cup Z]^a$ connecting two nodes $A$ and $B$ given $Z$. Then, a path in $G$ between $A$ and $B$ can be obtained as follows. First, replace every augmented edge on $\rho$ with an associated collider path in $G[A \cup B \cup Z]$. Second, replace every non-augmented edge on $\rho$ with an associated edge in $G[A \cup B \cup Z]$. Third, replace any configuration $C - D \la F \ra D - E$ produced in the previous steps with $C - D - E$.
\end{lemma}

\begin{proof}
We start by proving that the collider paths added in the first step of the lemma either do not have any node in common except possibly one of the endpoints, or the third step of the lemma removes the repeated nodes. Suppose for a contradiction that $C - D$ and $C' - D'$ are two augmented edges on $\rho$ such that their associated collider paths have in common a node which is not an endpoint of these paths. Consider the following two cases.

\begin{description}
\item[Case 1] Suppose that $D \neq C'$. Then, one of the following configurations must exist in $G[A \cup B \cup Z]$.
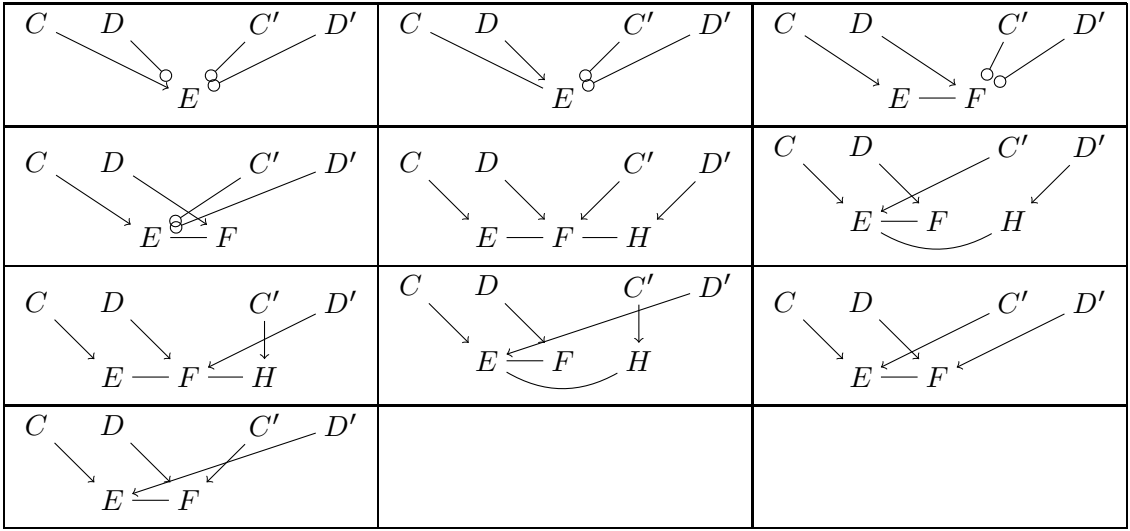
\begin{figure}[H]
\centering
\begin{tabular}{|c|c|c|}
\hline
\begin{tikzpicture}[inner sep=1mm]
\node at (0,0) (C) {$C$};
\node at (1,0) (D) {$D$};
\node at (2,-1) (E) {$E$};
\node at (3,0) (C') {$C'$};
\node at (4,0) (D') {$D'$};

\path[->] (C) edge (E);
\path[-o] (D) edge (E);
\path[-o] (C') edge (E);
\path[-o] (D') edge (E);
\end{tikzpicture}
&
\begin{tikzpicture}[inner sep=1mm]
\node at (0,0) (C) {$C$};
\node at (1,0) (D) {$D$};
\node at (2,-1) (E) {$E$};
\node at (3,0) (C') {$C'$};
\node at (4,0) (D') {$D'$};

\path[-] (C) edge (E);
\path[->] (D) edge (E);
\path[-o] (C') edge (E);
\path[-o] (D') edge (E);
\end{tikzpicture}
&
\begin{tikzpicture}[inner sep=1mm]
\node at (0,0) (C) {$C$};
\node at (1,0) (D) {$D$};
\node at (1.5,-1) (E) {$E$};
\node at (2.5,-1) (F) {$F$};
\node at (3,0) (C') {$C'$};
\node at (4,0) (D') {$D'$};

\path[->] (C) edge (E);
\path[->] (D) edge (F);
\path[-] (E) edge (F);
\path[-o] (C') edge (F);
\path[-o] (D') edge (F);
\end{tikzpicture}
\\
\hline
\begin{tikzpicture}[inner sep=1mm]
\node at (0,0) (C) {$C$};
\node at (1,0) (D) {$D$};
\node at (1.5,-1) (E) {$E$};
\node at (2.5,-1) (F) {$F$};
\node at (3,0) (C') {$C'$};
\node at (4,0) (D') {$D'$};

\path[->] (C) edge (E);
\path[->] (D) edge (F);
\path[-] (E) edge (F);
\path[-o] (C') edge (E);
\path[-o] (D') edge (E);
\end{tikzpicture}
&
\begin{tikzpicture}[inner sep=1mm]
\node at (0,0) (C) {$C$};
\node at (1,0) (D) {$D$};
\node at (1,-1) (E) {$E$};
\node at (2,-1) (F) {$F$};
\node at (3,-1) (H) {$H$};
\node at (3,0) (C') {$C'$};
\node at (4,0) (D') {$D'$};

\path[->] (C) edge (E);
\path[->] (D) edge (F);
\path[-] (E) edge (F);
\path[->] (C') edge (F);
\path[->] (D') edge (H);
\path[-] (F) edge (H);
\end{tikzpicture}
&
\begin{tikzpicture}[inner sep=1mm]
\node at (0,0) (C) {$C$};
\node at (1,0) (D) {$D$};
\node at (1,-1) (E) {$E$};
\node at (2,-1) (F) {$F$};
\node at (3,-1) (H) {$H$};
\node at (3,0) (C') {$C'$};
\node at (4,0) (D') {$D'$};

\path[->] (C) edge (E);
\path[->] (D) edge (F);
\path[-] (E) edge (F);
\path[->] (C') edge (E);
\path[->] (D') edge (H);
\path[-] (E) edge [bend right] (H);
\end{tikzpicture}
\\
\hline
\begin{tikzpicture}[inner sep=1mm]
\node at (0,0) (C) {$C$};
\node at (1,0) (D) {$D$};
\node at (1,-1) (E) {$E$};
\node at (2,-1) (F) {$F$};
\node at (3,-1) (H) {$H$};
\node at (3,0) (C') {$C'$};
\node at (4,0) (D') {$D'$};

\path[->] (C) edge (E);
\path[->] (D) edge (F);
\path[-] (E) edge (F);
\path[->] (C') edge (H);
\path[->] (D') edge (F);
\path[-] (F) edge (H);
\end{tikzpicture}
&
\begin{tikzpicture}[inner sep=1mm]
\node at (0,0) (C) {$C$};
\node at (1,0) (D) {$D$};
\node at (1,-1) (E) {$E$};
\node at (2,-1) (F) {$F$};
\node at (3,-1) (H) {$H$};
\node at (3,0) (C') {$C'$};
\node at (4,0) (D') {$D'$};

\path[->] (C) edge (E);
\path[->] (D) edge (F);
\path[-] (E) edge (F);
\path[->] (C') edge (H);
\path[->] (D') edge (E);
\path[-] (E) edge [bend right] (H);
\end{tikzpicture}
&
\begin{tikzpicture}[inner sep=1mm]
\node at (0,0) (C) {$C$};
\node at (1,0) (D) {$D$};
\node at (1,-1) (E) {$E$};
\node at (2,-1) (F) {$F$};
\node at (3,0) (C') {$C'$};
\node at (4,0) (D') {$D'$};

\path[->] (C) edge (E);
\path[->] (D) edge (F);
\path[-] (E) edge (F);
\path[->] (C') edge (E);
\path[->] (D') edge (F);
\end{tikzpicture}
\\
\hline
\begin{tikzpicture}[inner sep=1mm]
\node at (0,0) (C) {$C$};
\node at (1,0) (D) {$D$};
\node at (1,-1) (E) {$E$};
\node at (2,-1) (F) {$F$};
\node at (3,0) (C') {$C'$};
\node at (4,0) (D') {$D'$};

\path[->] (C) edge (E);
\path[->] (D) edge (F);
\path[-] (E) edge (F);
\path[->] (C') edge (F);
\path[->] (D') edge (E);
\end{tikzpicture}
&&\\
\hline
\end{tabular}
\end{figure}
However, the first case implies that $C - D'$ is in $G[A \cup B \cup Z]^a$, which implies that replacing the subpath of $\rho$ between $C$ and $D'$ with $C - D'$ results in a path in $G[A \cup B \cup Z]^a$ connecting $A$ and $B$ given $Z$ that is shorter than $\rho$. This is a contradiction. Similarly for the fourth, sixth, seventh, eighth, ninth and tenth cases. And similarly for the rest of the cases by replacing the subpath of $\rho$ between $D$ and $D'$ with $D - D'$.

\item[Case 2] Suppose that $D = C'$. Then, one of the following configurations must exist in $G[A \cup B \cup Z]$.
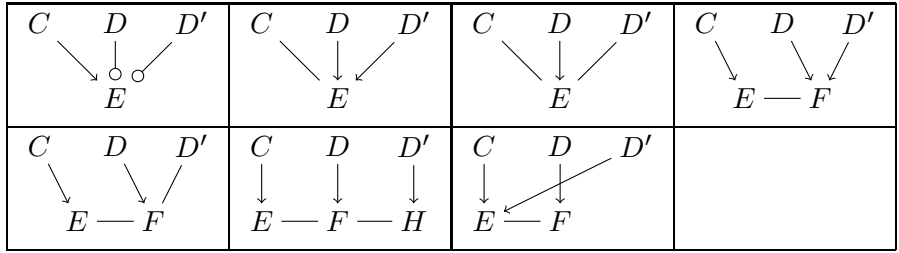
\begin{figure}[H]
\centering
\begin{tabular}{|c|c|c|c|}
\hline
\begin{tikzpicture}[inner sep=1mm]
\node at (0,0) (C) {$C$};
\node at (1,0) (D) {$D$};
\node at (1,-1) (E) {$E$};
\node at (2,0) (D') {$D'$};

\path[->] (C) edge (E);
\path[-o] (D) edge (E);
\path[-o] (D') edge (E);
\end{tikzpicture}
&
\begin{tikzpicture}[inner sep=1mm]
\node at (0,0) (C) {$C$};
\node at (1,0) (D) {$D$};
\node at (1,-1) (E) {$E$};
\node at (2,0) (D') {$D'$};

\path[-] (C) edge (E);
\path[->] (D) edge (E);
\path[->] (D') edge (E);
\end{tikzpicture}
&
\begin{tikzpicture}[inner sep=1mm]
\node at (0,0) (C) {$C$};
\node at (1,0) (D) {$D$};
\node at (1,-1) (E) {$E$};
\node at (2,0) (D') {$D'$};

\path[-] (C) edge (E);
\path[->] (D) edge (E);
\path[-] (D') edge (E);
\end{tikzpicture}
&
\begin{tikzpicture}[inner sep=1mm]
\node at (0,0) (C) {$C$};
\node at (1,0) (D) {$D$};
\node at (0.5,-1) (E) {$E$};
\node at (1.5,-1) (F) {$F$};
\node at (2,0) (D') {$D'$};

\path[->] (C) edge (E);
\path[->] (D) edge (F);
\path[-] (E) edge (F);
\path[->] (D') edge (F);
\end{tikzpicture}
\\
\hline
\begin{tikzpicture}[inner sep=1mm]
\node at (0,0) (C) {$C$};
\node at (1,0) (D) {$D$};
\node at (0.5,-1) (E) {$E$};
\node at (1.5,-1) (F) {$F$};
\node at (2,0) (D') {$D'$};

\path[->] (C) edge (E);
\path[->] (D) edge (F);
\path[-] (E) edge (F);
\path[-] (D') edge (F);
\end{tikzpicture}
&
\begin{tikzpicture}[inner sep=1mm]
\node at (0,0) (C) {$C$};
\node at (1,0) (D) {$D$};
\node at (0,-1) (E) {$E$};
\node at (1,-1) (F) {$F$};
\node at (2,-1) (H) {$H$};
\node at (2,0) (D') {$D'$};

\path[->] (C) edge (E);
\path[->] (D) edge (F);
\path[-] (E) edge (F);
\path[->] (D') edge (H);
\path[-] (F) edge (H);
\end{tikzpicture}
&
\begin{tikzpicture}[inner sep=1mm]
\node at (0,0) (C) {$C$};
\node at (1,0) (D) {$D$};
\node at (0,-1) (E) {$E$};
\node at (1,-1) (F) {$F$};
\node at (2,0) (D') {$D'$};

\path[->] (C) edge (E);
\path[->] (D) edge (F);
\path[-] (E) edge (F);
\path[->] (D') edge (E);
\end{tikzpicture}
&\\
\hline
\end{tabular}
\end{figure}
However, the first case implies that $C - D'$ is in $G[A \cup B \cup Z]^a$, which implies that replacing the subpath of $\rho$ between $C$ and $D'$ with $C - D'$ results in a path in $G[A \cup B \cup Z]^a$ connecting $A$ and $B$ given $Z$ that is shorter than $\rho$. This is a contradiction. Similarly for the second, fourth and seventh cases. For the third, fifth and sixth cases, the third step of the lemma removes the repeated nodes. Specifically, it replaces $C - E \la D \ra E - D'$ with $C - E - D'$ in the third case, $E - F \la D \ra F - D'$ with $E - F - D'$ in the fifth case, and $E - F \la D \ra F - H$ with $E - F - H$ in the sixth case.

\end{description}

It only remains to prove that the collider paths added in the first step of the lemma have no nodes in common with $\rho$ except the endpoints. Suppose that $\rho$ has an augmented edge $C - D$. Then, one of the following configurations must exist in $G[A \cup B \cup Z]$.
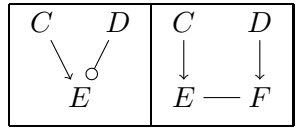
\begin{figure}[H]
\centering
\begin{tabular}{|c|c|}
\hline
\begin{tikzpicture}[inner sep=1mm]
\node at (0,0) (C) {$C$};
\node at (1,0) (D) {$D$};
\node at (0.5,-1) (E) {$E$};

\path[->] (C) edge (E);
\path[-o] (D) edge (E);
\end{tikzpicture}
&
\begin{tikzpicture}[inner sep=1mm]
\node at (0,0) (C) {$C$};
\node at (1,0) (D) {$D$};
\node at (0,-1) (E) {$E$};
\node at (1,-1) (F) {$F$};

\path[->] (C) edge (E);
\path[->] (D) edge (F);
\path[-] (E) edge (F);
\end{tikzpicture}\\
\hline
\end{tabular}
\end{figure}
Consider the first case and suppose for a contradiction that $E$ occurs on $\rho$. Note that $E \notin Z$ because, otherwise, $\rho$ would not be connecting. Assume without loss of generality that $E$ occurs on $\rho$ before $C$ and $D$ because, otherwise, a symmetric argument applies. Then, replacing the subpath of $\rho$ between $E$ and $D$ with $E - D$ results in a path in $G[A \cup B \cup Z]^a$ connecting $A$ and $B$ given $Z$ that is shorter than $\rho$. This is a contradiction. Similarly for the second case. Specifically, assume without loss of generality that $E$ occurs on $\rho$ because, otherwise, a symmetric argument with $F$ applies. Note that $E \notin Z$ because, otherwise, $\rho$ would not be connecting. If $E$ occurs on $\rho$ after $C$ and $D$, then replace the subpath of $\rho$ between $C$ and $E$ with $C - E$. This results in a path in $G[A \cup B \cup Z]^a$ connecting $A$ and $B$ given $Z$ that is shorter than $\rho$, which is a contradiction. If $E$ occurs on $\rho$ before $C$ and $D$, then replace the subpath of $\rho$ between $E$ and $D$ with $E - D$. This results in a path in $G[A \cup B \cup Z]^a$ connecting $A$ and $B$ given $Z$ that is shorter than $\rho$, which is a contradiction.
\end{proof}

\begin{lemma}\label{lem:noncolliders}
Let $\rho$ denote a path in an ADMG $G$ connecting two nodes $A$ and $B$ given $Z$. The sequence of non-colliders on $\rho$ forms a path in $G[A \cup B \cup Z]^a$ between $A$ and $B$.
\end{lemma}

\begin{proof}
Consider the maximal undirected subpaths of $\rho$. Note that each endpoint of each subpath is ancestor of a collider or endpoint of $\rho$, because $\rho$ is connecting. Thus, all the nodes on $\rho$ are in $G[A \cup B \cup Z]^a$. Suppose that $C$ and $D$ are two successive non-colliders on $\rho$. Then, the subpath of $\rho$ between $C$ and $D$ consists entirely of colliders. Specifically, the subpath is of the form $C \ot D$, $C \ra D$, $C \ra E \ot D$ or $C \ra E - F \la D$. Then, $C$ and $D$ are adjacent in $G[A \cup B \cup Z]^a$.
\end{proof}

\begin{proof}[\bf Proof of Theorem \ref{the:1}]
We start by proving the only if part. Let $\rho$ denote a path in $G$ connecting $A \in X$ and $B \in Y$ given $Z$. By Lemma \ref{lem:noncolliders} the non-colliders on $\rho$ form a path $\rho^a$ between $A$ and $B$ in $G[X \cup Y \cup Z]^a$. Since $\rho$ is connecting, every non-collider $C$ on $\rho$ is outside $Z$ unless $D - C - E$ is a subpath of $\rho$ and $Pa_G(C) \setminus Z \neq \emptyset$. In the latter case, $\rho^a$ has a subpath $D' - C - E'$ where $D'=D$ unless D is a collider on $\rho$, i.e. $D' \ra D - C$ is on $\rho$. Similarly for $E$ and $E'$. Therefore, we replace the subpath $D' - C - E'$ of $\rho^a$ with $D' - F - E'$ where $F \in Pa_G(C) \setminus Z$. Then, $\rho^a$ is connecting given $Z$. Note that $D' - F$ is in $G[X \cup Y \cup Z]^a$, because $D - C$ or $D' \ra D - C$ is in $\rho$ and $F \ra C$ is in $G$ with $C \in Z$. Similarly for $F - E'$.

To prove the if part, let $\rho^a$ denote a shortest path in $G[X \cup Y \cup Z]^a$ connecting $A \in X$ and $B \in Y$ given $Z$. We can transform $\rho^a$ into a path $\rho$ in $G$ as described in Lemma \ref{lem:colliders}. Since $\rho^a$ is connecting, no node on $\rho^a$ is in $Z$ and, thus, no non-collider on $\rho$ is in $Z$. Finally, since all the nodes on $\rho$ are in $G[X \cup Y \cup Z]^a$, it follows that every collider on $\rho$ is in $An_G(X \cup Y \cup Z)$. To see it, note that if $C - D$ is an augmented edge in $G[X \cup Y \cup Z]^a$ then the colliders on any collider path associated with $C - D$ are in $An_G(X \cup Y \cup Z)$. Thus, by Lemma \ref{lem:ancestors} there exist a node in $X$ and a node in $Y$ which are connected given $Z$ in $G$.
\end{proof}

\begin{proof}[\bf Proof of Theorem \ref{the:2}]
We prove the theorem for the following separation criterion, which is equivalent to criterion 2: A route is said to be connecting given $Z \subseteq V$ when
\begin{itemize}
\item every collider on the route is in $Z$, and

\item every non-collider $C$ on the route is outside $Z$ unless $A - C - B$ is a subroute and $Pa_G(C) \setminus Z \neq \emptyset$.
\end{itemize}

The only if part is trivial. To prove the if part, let $\rho$ denote a route in $G$ connecting $A$ and $B$ given $Z$. Let $C$ denote a node that occurs more than once in $\rho$. Consider the following cases.

\begin{description}
\item[Case 1] Assume that $\rho$ is of the form $A \ldots D \ra C \ldots C \ra E \ldots B$. Then, $C \notin Z$ for $\rho$ to be connecting given $Z$. Then, removing the subroute between the two occurrences of $C$ from $\rho$ results in the route $A \ldots D \ra C \ra E \ldots B$, which is connecting given $Z$.

\item[Case 2] Assume that $\rho$ is of the form $A \ldots D \ra C \ldots C \ot E \ldots B$. Then, $C \in An_G(Z)$ for $\rho$ to be connecting given $Z$. Then, removing the subroute between the two occurrences of $C$ from $\rho$ results in the route $A \ldots D \ra C \ot E \ldots B$, which is connecting given $Z$.

\item[Case 3] Assume that $\rho$ is of the form $A \ldots D \la C \ldots C \ldots B$. Then, $C \notin Z$ for $\rho$ to be connecting given $Z$. Then, removing the subroute between the two occurrences of $C$ from $\rho$ results in the route $A \ldots D \la C \ldots B$, which is connecting given $Z$.

\item[Case 4] Assume that $\rho$ is of the form $A \ldots D - C \ldots C \ra E \ldots B$. Then, $C \notin Z$ for $\rho$ to be connecting given $Z$. Then, removing the subroute between the two occurrences of $C$ from $\rho$ results in the route $A \ldots D - C \ra E \ldots B$, which is connecting given $Z$.

\item[Case 5] Assume that $\rho$ is of the form $A \ldots D - C \ldots C \la E \ldots B$. Then, $C \in An_G(Z)$ for $\rho$ to be connecting given $Z$. Then, removing the subroute between the two occurrences of $C$ from $\rho$ results in the route $A \ldots D - C \la E \ldots B$, which is connecting given $Z$.

\item[Case 6] Assume that $\rho$ is of the form $A \ldots D - C \ldots C - E \ldots B$ and $C \notin Z$. Then, removing the subroute between the two occurrences of $C$ from $\rho$ results in the route $A \ldots D - C - E \ldots B$, which is connecting given $Z$.

\item[Case 7] Assume that $\rho$ is of the form $A \ldots D - C \ldots C - E \ldots B$ and $C \in Z$. Then, $\rho$ must actually be of the form $A \ldots D - C \la F \ldots C - E \ldots B$ or $A \ldots D - C - F \ldots C - E \ldots B$. Note that in the former case $F \notin Z$ for $\rho$ to be connecting given $Z$. For the same reason, $Pa_G(C) \setminus Z \neq \emptyset$ in the latter case. Then, $Pa_G(C) \setminus Z \neq \emptyset$ in either case. Then, removing the subroute between the two occurrences of $C$ from $\rho$ results in the route $A \ldots D - C - E \ldots B$, which is connecting given $Z$.
\end{description}

Repeating the process above until no such node $C$ exists produces the desired path.
\end{proof}

\begin{proof}[\bf Proof of Theorem \ref{the:3}]
We start by proving the only if part. Suppose that there is path in $G[X \cup Y \cup Z]^a$ connecting a node in $X$ and a node in $Y$ given $Z$. We can then obtain a path in $G$ connecting $A \in X$ and $B \in Y$ given $Z$ as shown in the proof of Theorem \ref{the:1}. In this path, replace with $C - D$ every subpath $C - V_1 - \ldots - V_n - D$ such that $C, D \in An_G(X \cup Y \cup Z)$ and $V_1, \ldots, V_n \notin An_G(X \cup Y \cup Z)$. The result is a path in $G[X \cup Y \cup Z]^m$. Moreover, the path connects $A$ and $B$ given $Z$. To see it, note that the resulting and original paths have the same colliders, and the non-colliders on the resulting path are a subset of the non-colliders on the original path. Then, there is path in $(G[X \cup Y \cup Z]^m)^a$ connecting $A$ and $B$ given $Z$.

To prove the if part, suppose that there is path in $(G[X \cup Y \cup Z]^m)^a$ connecting $A \in X$ and $B \in Y$ given $Z$. Suppose that the path contains an edge $C - D$ that is not in $G[X \cup Y \cup Z]^a$. This is due to one the following reasons.

\begin{description}
\item[Case 1] $C - V_1 - \ldots - V_n - D$ is in $G[X \cup Y \cup Z]$ with $V_1, \ldots, V_n \notin An_G(X \cup Y \cup Z)$. Then, $C - V_1 - \ldots - V_n - D$ is in $G[X \cup Y \cup Z]^a$.

\item[Case 2] $C \ra E - D$ is in $G[X \cup Y \cup Z]^m$, which means that $C \ra E - V_1 - \ldots - V_n - D$ is in $G[X \cup Y \cup Z]$ with $V_1, \ldots, V_n \notin An_G(X \cup Y \cup Z)$. Then, $C - V_1 - \ldots - V_n - D$ is in $G[X \cup Y \cup Z]^a$.

\item[Case 3] $C \ra E - F \la D$ is in $G[X \cup Y \cup Z]^m$, which means that $C \ra E - V_1 - \ldots - V_n - F \la D$ is in $G[X \cup Y \cup Z]$ with $V_1, \ldots, V_n \notin An_G(X \cup Y \cup Z)$. Then, $C - V_1 - \ldots - V_n - D$ is in $G[X \cup Y \cup Z]^a$.
\end{description}

Either case implies that there is a path in $G[X \cup Y \cup Z]^a$ connecting $A \in X$ and $B \in Y$ given $Z$.
\end{proof}

\begin{proof}[\bf Proof of Theorem \ref{the:local}]
We start by proving the only if part. It suffices to note that every node that is adjacent to $B$ in $G[S]^a$ is in $Mb_{G[S]}(B)$, hence $B$ is separated from $S \setminus ( B \cup Mb_{G[S]}(B) )$ given $Mb_{G[S]}(B)$ in $G[S]^a$. Thus, $B \ci_p S \setminus ( B \cup Mb_{G[S]}(B) ) | Mb_{G[S]}(B)$ by the global Markov property.

To prove the if part, let $A$ be the node in $X \cup Y \cup Z$ that occurs the latest in $\prec$, and let $S = X \cup Y \cup Z$. Note that for all $B \in S$, the set of nodes that are adjacent to $B$ in $G[S]^a$ is precisely $Mb_{G[S]}(B)$. Then, the ordered local Markov property implies the global Markov property \citep[Theorem 3.7]{Lauritzen1996}.
\end{proof}

\begin{proof}[\bf Proof of Theorem \ref{the:pairwise}]
We start by proving the only if part. It suffices to note that if $B$ and $C$ are not adjacent in $G[S]^a$, then they are separated in $G[S]^a$ given $V(G[S]) \setminus ( B \cup C )$. Thus, $B \ci_p C | V(G[S]) \setminus ( B \cup C )$ by the global Markov property.

To prove the if part, let $A$ be the node in $X \cup Y \cup Z$ that occurs the latest in $\prec$, and let $S = X \cup Y \cup Z$. Then, the ordered pairwise Markov property implies the global Markov property \citep[Theorem 3.7]{Lauritzen1996}.
\end{proof}

\begin{proof}[\bf Proof of Theorem \ref{the:blockamp}]
First, C1$^*$ implies C3$^*$ by decomposition. Second, C1$^*$ implies C1 by taking $D=C$ and applying weak union. Third, C1 and the fact that $Nd_G(D) = Nd_G(C)$ imply $D \ci_p Nd_G(D) \setminus Cc_G(Pa_G(C)) | Cc_G(Pa_G(C))$ by symmetry and decomposition, which together with C3$^*$ imply C1$^*$ by contraction. Finally, C2 and C2$^*$ are equivalent because $p(C | Pa_G(C)) = p(C | Cc_G(Pa_G(C)))$ by C1$^*$ and decomposition.
\end{proof}

\begin{proof}[\bf Proof of Theorem \ref{the:localamp}]
To see the only if part, note that C1$^*$ with $D = C$ implies that $p(C | Nd_G(C)) = p(C | Pa_G(C))$. This implies that $p(C | Nd_G(C))$ satisfies the global Markov property with respect to $G_C$ by C2$^*$. This implies L1 \citep[Theorem 3.7]{Lauritzen1996}. Moreover, let $D = A \cup S$ and note that $Nd_G(D) = Nd_G(C)$. Then, L2$^*$ follows from C1$^*$ by symmetry and weak union.

To see the if part, note that L2$^*$ with $S = Ne_G(A)$ implies that $A \ci_p Nd_G(C) \setminus Pa_G(C) | Ne_G(A) \cup Pa_G(C)$ by weak union. This together with L1 imply $A \ci_p C \setminus (A \cup Ne_G(A)) | Ne_G(A) \cup Pa_G(C)$ by contraction and decomposition. This implies C2$^*$ \citep[Theorem 3.7]{Lauritzen1996}. Moreover, let $D = \{D_1, \ldots, D_n\}$. Then
\begin{enumerate}
\item $D_1 \ci_p Nd_G(C) \setminus Pa_G(D) | (D \setminus D_1) \cup Pa_G(D)$ by L2$^*$ with $A = D_1$ and $S = D \setminus D_1$.\label{eq:1}

\item $D_2 \ci_p Nd_G(C) \setminus Pa_G(D) | (D \setminus D_2) \cup Pa_G(D)$ by L2$^*$ with $A = D_2$ and $S = D \setminus D_2$.\label{eq:2}

\item $D_1 \cup D_2 \ci_p Nd_G(C) \setminus Pa_G(D) | (D \setminus (D_1 \cup D_2)) \cup Pa_G(D)$ by symmetry and intersection on (\ref{eq:1}) and (\ref{eq:2}).\label{eq:3}

\item $D_3 \ci_p Nd_G(C) \setminus Pa_G(D) | (D \setminus D_3) \cup Pa_G(D)$ by L2$^*$ with $A = D_3$ and $S = D \setminus D_3$.\label{eq:4}

\item $D_1 \cup D_2 \cup D_3 \ci_p Nd_G(C) \setminus Pa_G(D) | (D \setminus (D_1 \cup D_2 \cup D_3)) \cup Pa_G(D)$ by symmetry and intersection on (\ref{eq:3}) and (\ref{eq:4}).
\end{enumerate}
Continuing with this for $D_4, \ldots, D_n$ leads to C1$^*$.
\end{proof}

\begin{proof}[\bf Proof of Theorem \ref{the:pairwiseamp}]
To see the only if part, note that L1 and L2$^*$ imply P1 and P2$^*$ by weak union. To see the if part, let $Nd_G(C) \setminus Pa_G(A \cup S) = \{B_1, \ldots, B_n\}$. Then
\begin{enumerate}
\item $A \ci_p B_1 | S \cup Nd_G(C) \setminus B_1$ by P2$^*$ with $B = B_1$.\label{eq:1b}

\item $A \ci_p B_2 | S \cup Nd_G(C) \setminus B_2$ by P2$^*$ with $B = B_2$.\label{eq:2b}

\item $A \ci_p B_1 \cup B_2 | S \cup Nd_G(C) \setminus (B_1 \cup B_2)$ by intersection on (\ref{eq:1b}) and (\ref{eq:2b}).\label{eq:3b}

\item $A \ci_p B_3 | S \cup Nd_G(C) \setminus B_3$ by P2$^*$ with $B = B_3$.\label{eq:4b}

\item $A \ci_p B_1 \cup B_2 \cup B_3 | S \cup Nd_G(C) \setminus (B_1 \cup B_2 \cup B_3)$ by intersection on (\ref{eq:3b}) and (\ref{eq:4b}).
\end{enumerate}
Continuing with this for $B_4, \ldots, B_n$ leads to L2$^*$. Finally, let $C \setminus (A \cup Ne_G(A)) = \{B_1, \ldots, B_n\}$. Then
\begin{enumerate}\setcounter{enumi}{5}
\item $A \ci_p B_1 | Nd_G(C) \cup C \setminus (A \cup B_1)$ by P1 with $B = B_1$.\label{eq:1c}

\item $A \ci_p B_2 | Nd_G(C) \cup C \setminus (A \cup B_2)$ by P1 with $B = B_2$.\label{eq:2c}

\item $A \ci_p B_1 \cup B_2 | Nd_G(C) \cup C \setminus (A \cup B_1 \cup B_2)$ by intersection on (\ref{eq:1c}) and (\ref{eq:2c}).\label{eq:3c}

\item $A \ci_p B_3 | Nd_G(C) \cup C \setminus (A \cup B_3)$ by P1 with $B = B_3$.\label{eq:4c}

\item $A \ci_p B_1 \cup B_2 \cup B_3 | Nd_G(C) \setminus (B_1 \cup B_2 \cup B_3)$ by intersection on (\ref{eq:3c}) and (\ref{eq:4c}).
\end{enumerate}
Continuing with this for $B_4, \ldots, B_n$ leads to L1.
\end{proof}

\begin{proof}[{\bf Proof of Theorem \ref{the:GG'}}]
It suffices to show that every path in $G$ connecting $\alpha$ and $\beta$ given $Z$ can be transformed into a path in $G'$ connecting $\alpha$ and $\beta$ given $Z$ and vice versa, with $\alpha, \beta \in V$ and $Z \subseteq V \setminus ( \alpha \cup \beta )$.

Let $\rho$ denote a path in $G$ connecting $\alpha$ and $\beta$ given $Z$. We can easily transform $\rho$ into a path $\rho'$ in $G'$ between $\alpha$ and $\beta$: Simply, replace every maximal subpath of $\rho$ of the form $V_1 - V_2 - \ldots - V_{n-1} - V_n$ ($n \geq 2$) with $V_1 \la \epsilon_{V_1} - \epsilon_{V_2} - \ldots - \epsilon_{V_{n-1}} - \epsilon_{V_n} \ra V_n$. We now show that $\rho'$ is connecting given $Z$.

First, if $B \in V$ is a collider on $\rho'$, then $\rho'$ must have one of the following subpaths:

\begin{figure}[H]
\centering
\begin{tabular}{|c|c|c|}
\hline
\begin{tikzpicture}[inner sep=1mm]
\node at (0,0) (A) {$A$};
\node at (1,0) (B) {$B$};
\node at (2,0) (C) {$C$};
\path[->] (A) edge (B);
\path[<-] (B) edge (C);
\end{tikzpicture}
&
\begin{tikzpicture}[inner sep=1mm]
\node at (0,0) (A) {$A$};
\node at (1,0) (B) {$B$};
\node at (2,0) (C) {$\epsilon_B$};
\node at (3,0) (D) {$\epsilon_C$};
\path[->] (A) edge (B);
\path[<-] (B) edge (C);
\path[-] (D) edge (C);
\end{tikzpicture}
&
\begin{tikzpicture}[inner sep=1mm]
\node at (0,0) (A) {$\epsilon_B$};
\node at (1,0) (B) {$B$};
\node at (2,0) (C) {$C$};
\node at (-1,0) (D) {$\epsilon_A$};
\path[->] (A) edge (B);
\path[<-] (B) edge (C);
\path[-] (D) edge (A);
\end{tikzpicture}\\
\hline
\end{tabular}
\end{figure}

with $A, C \in V$. Therefore, $\rho$ must have one of the following subpaths (specifically, if $\rho'$ has the $i$-th subpath above, then $\rho$ has the $i$-th subpath below):

\begin{figure}[H]
\centering
\begin{tabular}{|c|c|c|}
\hline
\begin{tikzpicture}[inner sep=1mm]
\node at (0,0) (A) {$A$};
\node at (1,0) (B) {$B$};
\node at (2,0) (C) {$C$};
\path[->] (A) edge (B);
\path[<-] (B) edge (C);
\end{tikzpicture}
&
\begin{tikzpicture}[inner sep=1mm]
\node at (0,0) (A) {$A$};
\node at (1,0) (B) {$B$};
\node at (2,0) (C) {$C$};
\path[->] (A) edge (B);
\path[-] (B) edge (C);
\end{tikzpicture}
&
\begin{tikzpicture}[inner sep=1mm]
\node at (0,0) (A) {$A$};
\node at (1,0) (B) {$B$};
\node at (2,0) (C) {$C$};
\path[-] (A) edge (B);
\path[<-] (B) edge (C);
\end{tikzpicture}\\
\hline
\end{tabular}
\end{figure}

In either case, $B$ is a collider on $\rho$ and, thus, $B \in An_G(Z)$ for $\rho$ to be connecting given $Z$. Then, $B \in An_{G'}(Z)$ by construction of $G'$ and, thus, $B \in An_{G'}(Dt(Z))$.

Second, if $B \in V$ is a non-collider on $\rho'$, then $\rho'$ must have one of the following subpaths:

\begin{figure}[H]
\centering
\begin{tabular}{|c|c|c|c|c|}
\hline
\begin{tikzpicture}[inner sep=1mm]
\node at (0,0) (A) {$A$};
\node at (1,0) (B) {$B$};
\node at (2,0) (C) {$C$};
\path[->] (A) edge (B);
\path[->] (B) edge (C);
\end{tikzpicture}
&
\begin{tikzpicture}[inner sep=1mm]
\node at (0,0) (A) {$A$};
\node at (1,0) (B) {$B$};
\node at (2,0) (C) {$C$};
\path[<-] (A) edge (B);
\path[->] (B) edge (C);
\end{tikzpicture}
&
\begin{tikzpicture}[inner sep=1mm]
\node at (0,0) (A) {$A$};
\node at (1,0) (B) {$B$};
\node at (2,0) (C) {$C$};
\path[<-] (A) edge (B);
\path[<-] (B) edge (C);
\end{tikzpicture}
&
\begin{tikzpicture}[inner sep=1mm]
\node at (0,0) (A) {$A$};
\node at (1,0) (B) {$B$};
\node at (2,0) (C) {$\epsilon_B$};
\node at (3,0) (D) {$\epsilon_C$};
\path[<-] (A) edge (B);
\path[<-] (B) edge (C);
\path[-] (D) edge (C);
\end{tikzpicture}
&
\begin{tikzpicture}[inner sep=1mm]
\node at (0,0) (A) {$\epsilon_B$};
\node at (1,0) (B) {$B$};
\node at (2,0) (C) {$C$};
\node at (-1,0) (D) {$\epsilon_A$};
\path[->] (A) edge (B);
\path[->] (B) edge (C);
\path[-] (D) edge (A);
\end{tikzpicture}\\
\hline
\end{tabular}
\end{figure}

with $A, C \in V$. Therefore, $\rho$ must have one of the following subpaths (specifically, if $\rho'$ has the $i$-th subpath above, then $\rho$ has the $i$-th subpath below):

\begin{figure}[H]
\centering
\begin{tabular}{|c|c|c|c|c|}
\hline
\begin{tikzpicture}[inner sep=1mm]
\node at (0,0) (A) {$A$};
\node at (1,0) (B) {$B$};
\node at (2,0) (C) {$C$};
\path[->] (A) edge (B);
\path[->] (B) edge (C);
\end{tikzpicture}
&
\begin{tikzpicture}[inner sep=1mm]
\node at (0,0) (A) {$A$};
\node at (1,0) (B) {$B$};
\node at (2,0) (C) {$C$};
\path[<-] (A) edge (B);
\path[->] (B) edge (C);
\end{tikzpicture}
&
\begin{tikzpicture}[inner sep=1mm]
\node at (0,0) (A) {$A$};
\node at (1,0) (B) {$B$};
\node at (2,0) (C) {$C$};
\path[<-] (A) edge (B);
\path[<-] (B) edge (C);
\end{tikzpicture}
&
\begin{tikzpicture}[inner sep=1mm]
\node at (0,0) (A) {$A$};
\node at (1,0) (B) {$B$};
\node at (2,0) (C) {$C$};
\path[<-] (A) edge (B);
\path[-] (B) edge (C);
\end{tikzpicture}
&
\begin{tikzpicture}[inner sep=1mm]
\node at (0,0) (A) {$A$};
\node at (1,0) (B) {$B$};
\node at (2,0) (C) {$C$};
\path[-] (A) edge (B);
\path[->] (B) edge (C);
\end{tikzpicture}\\
\hline
\end{tabular}
\end{figure}

In either case, $B$ is a non-collider on $\rho$ and, thus, $B \notin Z$ for $\rho$ to be connecting given $Z$. Since $Z$ contains no error node, $Z$ cannot determine any node in $V$ that is not already in $Z$. Then, $B \notin Dt(Z)$.

Third, if $\epsilon_B$ is a non-collider on $\rho'$ (note that $\epsilon_B$ cannot be a collider on $\rho'$), then $\rho'$ must have one of the following subpaths:

\begin{figure}[H]
\centering
\begin{tabular}{|c|c|c|c|}
\hline
\begin{tikzpicture}[inner sep=1mm]
\node at (0,0) (A) {$A$};
\node at (1,0) (B) {$B$};
\node at (2,0) (C) {$\epsilon_B$};
\node at (3,0) (D) {$\epsilon_C$};
\path[->] (A) edge (B);
\path[<-] (B) edge (C);
\path[-] (D) edge (C);
\end{tikzpicture}
&
\begin{tikzpicture}[inner sep=1mm]
\node at (0,0) (A) {$\epsilon_B$};
\node at (1,0) (B) {$B$};
\node at (2,0) (C) {$C$};
\node at (-1,0) (D) {$\epsilon_A$};
\path[->] (A) edge (B);
\path[<-] (B) edge (C);
\path[-] (D) edge (A);
\end{tikzpicture}
&
\begin{tikzpicture}[inner sep=1mm]
\node at (0.65,0) (B) {$\alpha=B$};
\node at (2,0) (C) {$\epsilon_B$};
\node at (3,0) (D) {$\epsilon_C$};
\path[<-] (B) edge (C);
\path[-] (D) edge (C);
\end{tikzpicture}
&
\begin{tikzpicture}[inner sep=1mm]
\node at (0,0) (A) {$\epsilon_B$};
\node at (1.35,0) (B) {$B=\beta$};
\node at (-1,0) (D) {$\epsilon_A$};
\path[->] (A) edge (B);
\path[-] (D) edge (A);
\end{tikzpicture}
\\
\hline
\begin{tikzpicture}[inner sep=1mm]
\node at (0,0) (A) {$A$};
\node at (1,0) (B) {$B$};
\node at (2,0) (C) {$\epsilon_B$};
\node at (3,0) (D) {$\epsilon_C$};
\path[<-] (A) edge (B);
\path[<-] (B) edge (C);
\path[-] (D) edge (C);
\end{tikzpicture}
&
\begin{tikzpicture}[inner sep=1mm]
\node at (0,0) (A) {$\epsilon_B$};
\node at (1,0) (B) {$B$};
\node at (2,0) (C) {$C$};
\node at (-1,0) (D) {$\epsilon_A$};
\path[->] (A) edge (B);
\path[->] (B) edge (C);
\path[-] (D) edge (A);
\end{tikzpicture}
&
\begin{tikzpicture}[inner sep=1mm]
\node at (0,0) (A) {$\epsilon_A$};
\node at (1,0) (B) {$\epsilon_B$};
\node at (2,0) (C) {$\epsilon_C$};
\path[-] (A) edge (B);
\path[-] (B) edge (C);
\end{tikzpicture}
&\\
\hline
\end{tabular}
\end{figure}

with $A, C \in V$. Recall that $\epsilon_B \notin Z$ because $Z \subseteq V \setminus ( \alpha \cup \beta )$. In the first case, if $\alpha=A$ then $A \notin Z$, else $A \notin Z$ for $\rho$ to be connecting given $Z$. Then, $\epsilon_B \notin Dt(Z)$. In the second case, if $\beta=C$ then $C \notin Z$, else $C \notin Z$ for $\rho$ to be connecting given $Z$. Then, $\epsilon_B \notin Dt(Z)$. In the third and fourth cases, $B \notin Z$ because $\alpha=B$ or $\beta=B$. Then, $\epsilon_B \notin Dt(Z)$. In the fifth and sixth cases, $B \notin Z$ for $\rho$ to be connecting given $Z$. Then, $\epsilon_B \notin Dt(Z)$. The last case implies that $\rho$ has the following subpath:

\begin{figure}[H]
\centering
\begin{tabular}{|c|}
\hline
\begin{tikzpicture}[inner sep=1mm]
\node at (0,0) (A) {$A$};
\node at (1,0) (B) {$B$};
\node at (2,0) (C) {$C$};
\path[-] (A) edge (B);
\path[-] (B) edge (C);
\end{tikzpicture}\\
\hline
\end{tabular}
\end{figure}

Thus, $B$ is a non-collider on $\rho$, which implies that $B \notin Z$ or $Pa_G(B) \setminus Z \neq \emptyset$ for $\rho$ to be connecting given $Z$. In either case, $\epsilon_B \notin Dt(Z)$ (recall that $Pa_{G'}(B)=Pa_G(B) \cup \epsilon_B$ by construction of $G'$).

Finally, let $\rho'$ denote a path in $G'$ connecting $\alpha$ and $\beta$ given $Z$. We can easily transform $\rho'$ into a path $\rho$ between $\alpha$ and $\beta$: Simply, replace every maximal subpath of $\rho'$ of the form $V_1 \la \epsilon_{V_1} - \epsilon_{V_2} - \ldots - \epsilon_{V_{n-1}} - \epsilon_{V_n} \ra V_n$ ($n \geq 2$) with $V_1 - V_2 - \ldots - V_{n-1} - V_n$. We now show that $\rho$ is connecting given $Z$.

First, note that all the nodes in $\rho$ are in $V$. Moreover, if $B$ is a collider on $\rho$, then $\rho$ must have one of the following subpaths:

\begin{figure}[H]
\centering
\begin{tabular}{|c|c|c|}
\hline
\begin{tikzpicture}[inner sep=1mm]
\node at (0,0) (A) {$A$};
\node at (1,0) (B) {$B$};
\node at (2,0) (C) {$C$};
\path[->] (A) edge (B);
\path[<-] (B) edge (C);
\end{tikzpicture}
&
\begin{tikzpicture}[inner sep=1mm]
\node at (0,0) (A) {$A$};
\node at (1,0) (B) {$B$};
\node at (2,0) (C) {$C$};
\path[->] (A) edge (B);
\path[-] (B) edge (C);
\end{tikzpicture}
&
\begin{tikzpicture}[inner sep=1mm]
\node at (0,0) (A) {$A$};
\node at (1,0) (B) {$B$};
\node at (2,0) (C) {$C$};
\path[-] (A) edge (B);
\path[<-] (B) edge (C);
\end{tikzpicture}\\
\hline
\end{tabular}
\end{figure}

with $A, C \in V$. Therefore, $\rho'$ must have one of the following subpaths (specifically, if $\rho$ has the $i$-th subpath above, then $\rho'$ has the $i$-th subpath below):

\begin{figure}[H]
\centering
\begin{tabular}{|c|c|c|}
\hline
\begin{tikzpicture}[inner sep=1mm]
\node at (0,0) (A) {$A$};
\node at (1,0) (B) {$B$};
\node at (2,0) (C) {$C$};
\path[->] (A) edge (B);
\path[<-] (B) edge (C);
\end{tikzpicture}
&
\begin{tikzpicture}[inner sep=1mm]
\node at (0,0) (A) {$A$};
\node at (1,0) (B) {$B$};
\node at (2,0) (C) {$\epsilon_B$};
\node at (3,0) (D) {$\epsilon_C$};
\path[->] (A) edge (B);
\path[<-] (B) edge (C);
\path[-] (D) edge (C);
\end{tikzpicture}
&
\begin{tikzpicture}[inner sep=1mm]
\node at (0,0) (A) {$\epsilon_B$};
\node at (1,0) (B) {$B$};
\node at (2,0) (C) {$C$};
\node at (-1,0) (D) {$\epsilon_A$};
\path[->] (A) edge (B);
\path[<-] (B) edge (C);
\path[-] (D) edge (A);
\end{tikzpicture}\\
\hline
\end{tabular}
\end{figure}

In either case, $B$ is a collider on $\rho'$ and, thus, $B \in An_{G'}(Dt(Z))$ for $\rho'$ to connecting given $Z$. Since $Z$ contains no error node, $Z$ cannot determine any node in $V$ that is not already in $Z$. Then, $B \in Dt(Z)$ iff $B \in Z$. Since no error node is a descendant of $B$, then any node $D \in Dt(Z)$ that is a descendant of $B$ must be in $V$ which, as seen, implies that $D \in Z$. Then, $B \in An_{G'}(Dt(Z))$ iff $B \in An_{G'}(Z)$. Moreover, $B \in An_{G'}(Z)$ iff $B \in An_{G}(Z)$ by construction of $G'$. These results together imply that $B \in An_{G}(Z)$.

Second, if $B$ is a non-collider on $\rho$, then $\rho$ must have one of the following subpaths:

\begin{figure}[H]
\centering
\begin{tabular}{|c|c|c|c|c|}
\hline
\begin{tikzpicture}[inner sep=1mm]
\node at (0,0) (A) {$A$};
\node at (1,0) (B) {$B$};
\node at (2,0) (C) {$C$};
\path[->] (A) edge (B);
\path[->] (B) edge (C);
\end{tikzpicture}
&
\begin{tikzpicture}[inner sep=1mm]
\node at (0,0) (A) {$A$};
\node at (1,0) (B) {$B$};
\node at (2,0) (C) {$C$};
\path[<-] (A) edge (B);
\path[->] (B) edge (C);
\end{tikzpicture}
&
\begin{tikzpicture}[inner sep=1mm]
\node at (0,0) (A) {$A$};
\node at (1,0) (B) {$B$};
\node at (2,0) (C) {$C$};
\path[<-] (A) edge (B);
\path[<-] (B) edge (C);
\end{tikzpicture}
&
\begin{tikzpicture}[inner sep=1mm]
\node at (0,0) (A) {$A$};
\node at (1,0) (B) {$B$};
\node at (2,0) (C) {$C$};
\path[<-] (A) edge (B);
\path[-] (B) edge (C);
\end{tikzpicture}
&
\begin{tikzpicture}[inner sep=1mm]
\node at (0,0) (A) {$A$};
\node at (1,0) (B) {$B$};
\node at (2,0) (C) {$C$};
\path[-] (A) edge (B);
\path[->] (B) edge (C);
\end{tikzpicture}
\\
\hline
\begin{tikzpicture}[inner sep=1mm]
\node at (0,0) (A) {$A$};
\node at (1,0) (B) {$B$};
\node at (2,0) (C) {$C$};
\path[-] (A) edge (B);
\path[-] (B) edge (C);
\end{tikzpicture}
&&&&\\
\hline
\end{tabular}
\end{figure}

with $A, C \in V$. Therefore, $\rho'$ must have one of the following subpaths (specifically, if $\rho$ has the $i$-th subpath above, then $\rho'$ has the $i$-th subpath below):

\begin{figure}[H]
\centering
\begin{tabular}{|c|c|c|c|c|}
\hline
\begin{tikzpicture}[inner sep=1mm]
\node at (0,0) (A) {$A$};
\node at (1,0) (B) {$B$};
\node at (2,0) (C) {$C$};
\path[->] (A) edge (B);
\path[->] (B) edge (C);
\end{tikzpicture}
&
\begin{tikzpicture}[inner sep=1mm]
\node at (0,0) (A) {$A$};
\node at (1,0) (B) {$B$};
\node at (2,0) (C) {$C$};
\path[<-] (A) edge (B);
\path[->] (B) edge (C);
\end{tikzpicture}
&
\begin{tikzpicture}[inner sep=1mm]
\node at (0,0) (A) {$A$};
\node at (1,0) (B) {$B$};
\node at (2,0) (C) {$C$};
\path[<-] (A) edge (B);
\path[<-] (B) edge (C);
\end{tikzpicture}
&
\begin{tikzpicture}[inner sep=1mm]
\node at (0,0) (A) {$A$};
\node at (1,0) (B) {$B$};
\node at (2,0) (C) {$\epsilon_B$};
\node at (3,0) (D) {$\epsilon_C$};
\path[<-] (A) edge (B);
\path[<-] (B) edge (C);
\path[-] (D) edge (C);
\end{tikzpicture}
&
\begin{tikzpicture}[inner sep=1mm]
\node at (0,0) (A) {$\epsilon_B$};
\node at (1,0) (B) {$B$};
\node at (2,0) (C) {$C$};
\node at (-1,0) (D) {$\epsilon_A$};
\path[->] (A) edge (B);
\path[->] (B) edge (C);
\path[-] (D) edge (A);
\end{tikzpicture}
\\
\hline
\begin{tikzpicture}[inner sep=1mm]
\node at (0,0) (A) {$\epsilon_A$};
\node at (1,0) (B) {$\epsilon_B$};
\node at (2,0) (C) {$\epsilon_C$};
\path[-] (A) edge (B);
\path[-] (B) edge (C);
\end{tikzpicture}
&&&&\\
\hline
\end{tabular}
\end{figure}

In the first five cases, $B$ is a non-collider on $\rho'$ and, thus, $B \notin Dt(Z)$ for $\rho'$ to be connecting given $Z$. Since $Z$ contains no error node, $Z$ cannot determine any node in $V$ that is not already in $Z$. Then, $B \notin Z$. In the last case, $\epsilon_B$ is a non-collider on $\rho'$ and, thus, $\epsilon_B \notin Dt(Z)$ for $\rho'$ to be connecting given $Z$. Then, $B \notin Z$ or $Pa_{G'}(B) \setminus ( \epsilon_B \cup Z ) \ \neq \emptyset$. Then, $B \notin Z$ or $Pa_{G}(B) \setminus Z \ \neq \emptyset$ (recall that $Pa_{G'}(B)=Pa_G(B) \cup \epsilon_B$ by construction of $G'$).
\end{proof}

\begin{proof}[\bf Proof of Theorem \ref{the:gaussian}]
Modify the equation $A = \beta_A Pa_G(A) + \epsilon_A$ by replacing each $B \in V$ in the right-hand side of the equation with the right-hand side of the equation of $B$, i.e. $\beta_B Pa_G(B) + \epsilon_B$. Since $G$ is directed acyclic, repeating this process results in a set of equations for the elements of $V$ whose right-hand sides are linear combinations of the elements of $\epsilon$. In other words, $V = \delta \epsilon$ with $\epsilon \sim \mathcal{N}(0, \Lambda)$. Then, $V \sim \mathcal{N}(0, \delta \Lambda \delta^T)$.
\end{proof}

\begin{proof}[\bf Proof of Theorem \ref{the:markovian}]
Equation (\ref{eq:equation}) implies for any $A \in V$ that
\[
A \ci_{p(V \cup \epsilon)} (V \cup \epsilon) \setminus (A \cup Pa_{G'}(A)) | Pa_{G'}(A)
\]
and thus
\begin{equation}\label{eq:equation3}
A \ci_{p(V \cup \epsilon)} Nd_{G'}(A) \setminus Pa_{G'}(A) | Pa_{G'}(A)
\end{equation}
by decomposition.

Moreover, Equation (\ref{eq:equation}) implies for any $\epsilon_C \in Cc(G'_{\epsilon})$ that
\begin{equation}\label{eq:equation4}
p(\epsilon_C \cup Nd_{G'}(\epsilon_C)) = \prod_{K \in Cc(G'_{\epsilon})} p(\epsilon_K) \prod_{A \in Nd_{G'}(\epsilon_C) \cap V} \beta_A Pa_{G'}(A) = p(\epsilon_C) p(Nd_{G'}(\epsilon_C))
\end{equation}
and thus
\[
\epsilon_C \ci_{p(V \cup \epsilon)} Nd_{G'}(\epsilon_C) | \emptyset
\]
and thus
\begin{equation}\label{eq:equation5}
\epsilon_A \ci_{p(V \cup \epsilon)} Nd_{G'}(\epsilon_C) | \epsilon_Z
\end{equation}
where $\epsilon_A \in \epsilon_C$ and $\epsilon_Z \subseteq \epsilon_C \setminus \epsilon_A$, by decomposition and weak union.

Finally, Equation (\ref{eq:equation2}) implies for any $\epsilon_A \in \epsilon_C$ and $\epsilon_C \in Cc(G'_{\epsilon})$ that
\[
\epsilon_A \ci_{p(V \cup \epsilon)} \epsilon_C \setminus (\epsilon_A \cup Ne_{G'}(\epsilon_A)) | Ne_{G'}(\epsilon_A)
\]
by \citet[Theorem 3.7 and Proposition 5.2]{Lauritzen1996}, and thus
\[
p(\epsilon_C) = h(\epsilon_A \cup Ne_{G'}(\epsilon_A)) k(\epsilon_C \setminus \epsilon_A)
\]
by \citet[Equation 3.6]{Lauritzen1996}. This together with Equation (\ref{eq:equation4}) imply that
\[
p(\epsilon_C \cup Nd_{G'}(\epsilon_C)) = h(\epsilon_A \cup Ne_{G'}(\epsilon_A)) k(\epsilon_C \setminus \epsilon_A) p(Nd_{G'}(\epsilon_C))
\]
and thus
\begin{equation}\label{eq:equation6}
\epsilon_A \ci_{p(V \cup \epsilon)} \epsilon_C \setminus (\epsilon_A \cup Ne_{G'}(\epsilon_A)) | Nd_{G'}(\epsilon_C) Ne_{G'}(\epsilon_A)
\end{equation}
by \citet[Equation 3.6]{Lauritzen1996}.

Consequently, Equations (\ref{eq:equation3}), (\ref{eq:equation5}) and (\ref{eq:equation6}) imply that $p(V \cup \epsilon)$ satisfies the global Markov property with respect to $G'$ by Theorem \ref{the:localamp} because (i) $G'$ is actually an AMP CG over $V \cup \epsilon$, (ii) $A$ is the only node in the connectivity component of $G'$ that contains $A$, and (iii) $\epsilon_C$ has no parents in $G'$. Then, $p(V)$ satisfies the global Markov property with respect to $G$, because $G$ and $G'$ represent the same separations over $V$ by Theorem \ref{the:GG'}.
\end{proof}


\begin{thebibliography}{}

\bibitem[Andersson et al., 2001]{Anderssonetal.2001}
Andersson, S. A., Madigan, D. and Perlman, M. D. Alternative Markov Properties for Chain Graphs. {\em Scandinavian Journal of Statistics}, 28:33-85, 2001.

\bibitem[Bareinboim and Tian, 2015]{BareinboimandTian2015}
Bareinboim, E. and Tian, J. Recovering Causal Effects From Selection Bias. In {\em Proceedings of the 29th AAAI Conference on Artificial Intelligence}, 3475-3481, 2015.

\bibitem[Cox and Wermuth, 1996]{CoxandWermuth1996}
Cox, D. R. and Wermuth, N. {\em Multivariate Dependencies - Models, Analysis and Interpretation}. Chapman \& Hall, 1996.

\bibitem[Evans and Richardson, 2013]{EvansandRichardson2013}
Evans, R. J. and Richardson, T. S. Marginal log-linear Parameters for Graphical Markov Models. {\em Journal of the Royal Statistical Society B}, 75:743-768, 2013.

\bibitem[Biere et al., 2009]{DBLP:series/faia/2009-185}
Biere, A., Heule, M., van Maaren, H. and Walsh, T. (editors). {\em Handbook of Satisfiability}. IOS Press, 2009.

\bibitem[Gebser et al., 2011]{DBLP:journals/aicom/GebserKKOSS11}
Gebser, M., Kaufmann, B., Kaminski, R., Ostrowski, M., Schaub, T., Schneider, M. Potassco: The Potsdam Answer Set Solving Collection. {\em AI Communications}, 24:107-124, 2011.

\bibitem[Gelfond, 1988]{gelfond_1988}
Gelfond, M. and Lifschitz, V. The Stable Model Semantics for Logic Programming. In {\em Proceedings of 5th Logic Programming Symposium}, 1070-1080, 1988.

\bibitem[Kang and Tian, 2009]{KangandTian2009}
Kang, C. and Tian, J. Markov Properties for Linear Causal Models with Correlated Errors. {\em Journal of Machine Learning Research}, 10:41-70, 2009.

\bibitem[Koster, 1999]{Koster1999}
Koster, J. T. A. On the Validity of the Markov Interpretation of Path Diagrams of Gaussian Structural Equations Systems with Correlated Errors. {\em Scandinavian Journal of Statistics}, 26:413-431, 1999.

\bibitem[Koster, 2002]{Koster2002}
Koster, J. T. A. Marginalizing and Conditioning in Graphical Models. {\em Bernoulli}, 8:817-840, 2002.

\bibitem[Lauritzen, 1996]{Lauritzen1996}
Lauritzen, S. L. {\em Graphical Models}. Oxford University Press, 1996.

\bibitem[Levitz et al., 2001]{Levitzetal.2001}
Levitz, M., Perlman M. D. and Madigan, D. Separation and Completeness Properties for AMP Chain Graph Markov Models. {\em The Annals of Statistics}, 29:1751-1784, 2001.

\bibitem[Niemel{\"a}, 1999]{DBLP:journals/amai/Niemela99}
Niemel{\"a}, I. Logic Programs with Stable Model Semantics as a Constraint Programming Paradigm. {\em Annals of Mathematics and Artificial Intelligence}, 25:241-273, 1999.

\bibitem[Pearl, 1995]{Pearl1995}
Pearl, J. Causal Diagrams for Empirical Research. {\em Biometrika}, 82:669-688, 1995.

\bibitem[Pearl, 2009]{Pearl2009}
Pearl, J. {\em Causality: Models, Reasoning, and Inference} (2nd ed.). Cambridge University Press, 2009.

\bibitem[Pearl and Wermuth, 1993]{PearlandWermuth1993}
Pearl, J. and Wermuth, N. When Can Association Graphs Admit a Causal Explanation ? In {\em Proceedings of the 4th International Workshop on Artificial Intelligence and Statistics}, 141-150, 1993.

\bibitem[Richardson, 2003]{Richardson2003}
Richardson, T. Markov Properties for Acyclic Directed Mixed Graphs. {\em Scandinavian Journal of Statistics}, 30:145-157, 2003.

\bibitem[Richardson and Spirtes, 2002]{RichardsonandSpirtes2002}
Richardson, T. and Spirtes, P. Ancestral Graph Markov Models. {\em The Annals of Statistics}, 30:962-1030, 2002.

\bibitem[Sadeghi and Lauritzen, 2014]{SadeghiandLauritzen2014}
Sadeghi, K. and Lauritzen, S. L. Markov Properties for Mixed Graphs. {\em Bernoulli}, 20:676-696, 2014.

\bibitem[Simons et al., 2002]{DBLP:journals/ai/SimonsNS02}
Simons, P., Niemel{\"a}, I. and Soininen, T. Extending and Implementing the Stable Model Semantics. {\em Artificial Intelligence}, 138:181-234, 2002.

\bibitem[Sonntag and Pe\~{n}a, 2015]{SonntagandPenna2015}
Sonntag, D. and Pe\~{n}a, J. M. Chain Graph Interpretations and their Relations Revisited. {\em International Journal of Approximate Reasoning}, 58:39-56, 2015.

\bibitem[Sonntag and Pe\~{n}a, 2016]{SonntagandPenna2016}
Sonntag, D. and Pe\~{n}a, J. M. On Expressiveness of the Chain Graph Interpretations. {\em International Journal of Approximate Reasoning}, 68:91-107, 2016.

\bibitem[Studen\'{y}, 2005]{Studeny2005}
Studen\'{y}, M. {\em Probabilistic Conditional Independence Structures}. Springer, 2005.

\end{thebibliography}
\end{document}